\documentclass{article}

\usepackage[final, nonatbib]{neurips_2020}

\usepackage[numbers,sort&compress]{natbib}
\usepackage{cite}
\usepackage[utf8]{inputenc} 
\usepackage[T1]{fontenc}    
\usepackage{hyperref}       
\usepackage{url}            
\usepackage{booktabs}       
\usepackage{amsfonts}       
\usepackage{nicefrac}       
\usepackage{microtype}      
\usepackage[dvipsnames]{xcolor}
\usepackage{hyperref}
\usepackage{tikz}
\usetikzlibrary{arrows, decorations.pathreplacing}
\usetikzlibrary{patterns}

\usepackage{paralist}
\hypersetup{
    colorlinks=true,
    allcolors=MidnightBlue
}

\usepackage[shortlabels]{enumitem}

\usepackage{amsmath,amssymb,amsthm}
\usepackage{algorithm}
\usepackage[noend]{algorithmic}
\usepackage{subfigure}
\usepackage{mathtools}
\usepackage{bm}
\usepackage{soul}
\newtheorem{corollary}{Corollary}
\newtheorem{thm}{Theorem}
\newtheorem{lemma}{Lemma} 
\newtheorem{remark}{Remark}

\newcommand{\argmax}{\mathop{\mathrm{argmax}}\limits}
\newcommand{\expect}[2]{{\ensuremath{\mathbb{E}_{#1}\left[{#2}\right]}}}
\newcommand{\proj}{\ensuremath{\Pi}}
\newcommand{\pseudo}{\ensuremath{\bar \Sigma}}
\newcommand{\candidates}{\ensuremath{T}}
\newcommand{\name}{\text{OST}}
\newcommand\numberthis{\addtocounter{equation}{1}\tag{\theequation}}

\title{Learning Kernel Tests Without Data Splitting}

\author{%
  Jonas M. K\"ubler\\
  MPI for Intelligent Systems, T\"ubingen\\
  \texttt{jmkuebler@tue.mpg.de} \\
   \And
   Wittawat Jitkrittum\thanks{Now with Google Research} \\
   MPI for Intelligent Systems, T\"ubingen \\
   \texttt{wittawatj@gmail.com} \\
   \AND
   Bernhard Sch\"olkopf \\
   MPI for Intelligent Systems, T\"ubingen \\
   \texttt{bs@tue.mpg.de} \\
   \And
   Krikamol Muandet \\
   MPI for Intelligent Systems, T\"ubingen \\
   \texttt{krikamol@tue.mpg.de} \\
}

\begin{document}

\maketitle

\begin{abstract}
    Modern large-scale kernel-based tests such as maximum mean
    discrepancy (MMD) and kernelized Stein discrepancy (KSD) 
    optimize kernel hyperparameters on a held-out sample via data splitting to obtain the most powerful test statistics.
    While data splitting results in a
    tractable null distribution, it suffers from a reduction in test power due to
    smaller test sample size.
    Inspired by the selective inference framework, we propose an approach that enables learning the hyperparameters and
    testing on the full sample without 
    data splitting.
    Our approach can correctly calibrate the test in the
    presence of such dependency, and yield a test threshold in closed form.
    At the same significance level, our
    approach's test power is empirically larger than that of the data-splitting approach, regardless
    of its split 
    proportion. 
\end{abstract}

\section{Introduction}\label{sec:intro}

Statistical hypothesis testing is a ubiquitous problem in numerous fields ranging from astronomy and high-energy physics to medicine and psychology \citep{Neyman33:Testing}.
Given a hypothesis about a natural phenomenon, it prescribes a systematic way 
to test 
the hypothesis
empirically \citep{Lehmann05:Testing}.
Two-sample testing, for instance, 
addresses whether 
two 
samples 
originate
from the 
same
process, which is instrumental in experimental science such as psychology, medicine, and economics.
This 
procedure
of rejecting false hypotheses while retaining the correct ones governs most 
advances
in science.

Traditionally, test statistics are usually fixed prior to the testing phase. 
In modern-day hypothesis testing, 
however, 
practitioners often face 
a large family of test statistics from which the best one must be selected before performing the test. 
For instance, the popular kernel-based two-sample tests \citep{gretton2012kernel,Gretton2012optimal} and goodness-of-fit tests \citep{Liu16:KSD,Chwialkowski16:KGFT} require the specification of a kernel function and its parameter values.
Abundant evidence suggests that finding good parameter values for these tests improves their performance in the testing phase \citep{sutherland2016generative, SceVar2019, Gretton2012optimal,JitSzaChwGre2016}. As a result, several approaches have recently been proposed to learn 
optimal tests directly from data using different techniques such as optimized kernels \citep{Gretton2012optimal,JitKanSanHaySch2018,JitXuSzaFukGre2017,JitSzaGre2017,JitSzaChwGre2016, liu2020learning}, classifier two-sample tests \citep{kim2016classification, LopOqu2017}, and deep neural networks \citep{CheClo2019,KirKhoKloLip2019}, to name a few.
In other words, the modern-day hypothesis testing has become a two-stage ``learn-then-test'' problem.

Special care must be taken in the subsequent testing when optimal tests are learned from data.
If the same data is used for both learning and testing, it becomes harder to derive the asymptotic null distribution because 
the selected 
test and the data are now dependent.
In this case, conducting the tests as if the test statistics are independent from the data leads to an uncontrollable false positive rate, see, e.g., our experimental results.
While permutation testing can be applied \citep{Fisher35:ExpDesign}, it is too computationally prohibitive for real-world applications.
Up to now, the most prevalent solution is \emph{data splitting}: the data is randomly split into two parts, of which the former is used for learning the test while the latter is used for testing.
Although data splitting is simple and in principle leads to the correct false positive rate, its downside is a potential loss of power.  

In this paper, we investigate the two-stage ``learn-then-test'' problem in the context of modern kernel-based tests \citep{gretton2012kernel,Gretton2012optimal,Liu16:KSD,Chwialkowski16:KGFT} where the choice of kernel function and its parameters play an important role.
The key question is \emph{whether it is possible to employ the full sample for both learning and testing phase without data splitting, while correctly calibrating the test in the presence of such dependency}.
We provide an affirmative answer if we learn the test from a vector of jointly normal base test statistics, e.g., the linear-time MMD estimates of multiple kernels. 
The empirical results suggest that, at the same significance level, the test power of our approach is larger than that of the data-splitting approach, regardless of the split proportion (cf. Section \ref{sec:experiments}).
The code for the experiments is available at \url{https://github.com/jmkuebler/tests-wo-splitting}.

\section{Preliminaries}
\label{sec:preliminaries}
We start with some background material on conventional hypothesis testing and review linear-time kernel two-sample tests. 
In what follows, we will use $[d] := \{1,\dots,d\}$ to denote the set of natural numbers up to $d \in \mathbb{N}$, $\bm\mu \geq \bm 0$ to denote that all entries of $\bm\mu \in \mathbb{R}^d$ are non-negative, $e_i$ to denote the $i$-th Cartesian unit vector, and $\|\cdot\| :=\|\cdot\|_2$.

\vspace{-5pt}
\paragraph{Statistical hypothesis testing.}
Let $Z$ be a random variable taking values in $\mathcal{Z}\subseteq\mathbb{R}^p$  distributed according to a distribution $P$. 
The goal of statistical hypothesis testing is to decide whether some \emph{null hypothesis} $H_0$ about $P$ can be rejected in favor of an \emph{alternative hypothesis} $H_A$  based on empirical data
\citep{Anderson03:Multivariate,Lehmann05:Testing}. Let $h$ be a real-valued function  such that $0 < \expect{}{h^2(Z)} < \infty$. In this work, we consider 
testing
the null hypothesis 
$
    H_0: \expect{}{h(Z)} = 0$ 
against the one-sided alternative hypothesis 
$H_A: \expect{}{h(Z)} > 0$ for reasons which will become clear later. 
To do so, we define the \emph{test statistic} $
    \tau(Z_n) = \frac{1}{n}\sum_{i=1}^n h(z_i)
$ as the empirical mean of $h$ based on a sample ${Z}_n := \{{z}_1, ..., {z}_n\}$  drawn i.i.d. from $P^n$. 
We reject $H_0$ if the observed test statistic $\hat{\tau}(Z_n)$ is \emph{significantly} larger than what we would expect if $H_0$ was true, i.e., if $P(\tau(Z_n) <\hat{\tau}(Z_n)|H_0) > 1- \alpha$. 
Here $\alpha$ is a \emph{significance level}
and controls the probability of incorrectly rejecting $H_0$ (Type-I error). 
For sufficiently large $n$ we can work with the asymptotic distribution of $\tau(Z_n)$, which is characterized by the Central Limit Theorem 
\citep{serfling1980approximation}.
\begin{lemma}\label{lem:asymptotic_normality}
Let $\mu := \expect{}{h(Z)}$ and $\sigma^2 := \text{Var}\left[h(Z)\right]$.
 Then, the test statistic converges in distribution to a Gaussian distribution, i.e.,
$
    \sqrt{n} (\tau(Z_n) - \mu) \overset{d}{\rightarrow} \mathcal{N}(0,  \sigma^2).$
\end{lemma}
Let $\Phi$ be the CDF of the standard normal and $\Phi^{-1}$ its inverse. 
We define the test threshold $t_\alpha = \sqrt{n} \sigma \Phi^{-1}(1-\alpha)$ as the $(1-\alpha)$-quantile of the null distribution so that $P\left(\tau(Z_n) < t_\alpha |H_0\right) = 1- \alpha$
and we reject $H_0$ simply if $\hat\tau ({Z}_n) > t_\alpha$. 
Besides correctly controlling the Type-I error, the test should also reject $H_0$ as often as possible when $P$ actually satisfies the alternative $H_A$. 
The probability of making a Type-II error is defined as $P\left(\tau(Z_n) < t_\alpha \,|\, H_A\right)$, i.e., the probability of failing to reject $H_0$ when it is
false. 
A powerful test has a small Type-II error while keeping the Type-I error at $\alpha$. 
Since Lemma \ref{lem:asymptotic_normality} holds for any $\mu$, and thus both under  null and alternative hypotheses,
the asymptotic probability of
a Type-II error is \citep{Gretton2012optimal}
\begin{align}\label{eq:asymptotic_power}
    P(\tau(Z_n) < t_\alpha \mid H_A) \approx \Phi\left(\Phi^{-1}(1-\alpha) - \frac{\mu \sqrt{n}}{\sigma} \right).
\end{align}
Since $\Phi$ is monotonic, this probability decreases with $\mu/\sigma$, which we interpret as a signal-to-noise ratio (SNR). It is therefore desirable to find test statistics with high SNR.

\vspace{-5pt}
\paragraph{Kernel two-sample testing.}
As an example that can be expressed in the above form we present kernel two-sample tests. Given two samples $X_n$ and $Y_n$ drawn from 
distributions $P$ and $Q$, the two-sample test aims to decide whether $P$ and $Q$ are different, i.e., $H_0 : P = Q$ and $H_A: P \neq Q$.
A popular test statistic for this problem is the maximum mean discrepancy
(MMD) of \citet{gretton2012kernel}, which is defined based on a positive definite kernel function
$k$ \citep{Scholkopf01:LKS}:
$
\text{MMD}^2[P,Q] 
= \mathbb{E}[k(x,x') + k(y,y') - k(x,y') - k(x',y)] 
    = \mathbb{E}[h(x,x',y,y')],
$
where $x,x'$ are independent draws from $P$, $y,y'$ are independent draws from $Q$, 
and  $h(x,x',y,y') := k(x, x') + k(y, y') - k(x, y') - k(y, x')$.
A minimum-variance unbiased estimator of $\text{MMD}^2$ is given by a second-order $U$-statistic \citep{serfling1980approximation}.
However, this estimator scales quadratically with the sample size, and the distribution under $H_0$ is not available in closed form. 
Thus it has to be 
simulated
either via a bootstrapping approach or via a permutation of the samples. For large sample size, the computational requirements become prohibitive \citep{gretton2012kernel}. In this work, we assume we are in this regime. To circumvent these computational burdens, 
\citet{gretton2012kernel} suggest a ``linear-time'' MMD estimate that scales
linearly with sample size and is asymptotically normally distributed under
both null and alternative hypotheses.
Specifically, let $X_{2n} = \{x_1,\ldots,x_{2n} \}$ and $Y_{2n} = \{y_1,\ldots,y_{2n} \}$, i.e., the samples are of the same (even) size. 
We can define $z_i := (x_i, x_{n+i}, y_i, y_{n+i})$ and $\tau(Z_n) := \frac{1}{n}\sum_{i=1}^n h(z_i)$ as the test statistic, which by Lemma \ref{lem:asymptotic_normality} is asymptotically normally distributed. 
Furthermore, if the kernel $k$ is characteristic \citep{Sriperumbudur2010}, it is guaranteed that $\text{MMD}^2(P, Q) = 0$ if  $P = Q$  and $\text{MMD}^2(P, Q) > 0$ otherwise. Therefore, a one-sided test is sufficient. 

Other
well-known examples 
are goodness-of-fit tests based on the kernelized Stein discrepancy (KSD), which  also has a linear time estimate \citep{Chwialkowski16:KGFT,Liu16:KSD}. In our experiments, we focus on the kernel two-sample test, but point out that our theoretical treatment in Section \ref{sec:learn-to-test} is more general and can be applied to other problems, e.g., KSD goodness-of-fit tests, but also beyond kernel methods.

\section{Selective hypothesis tests}
\label{sec:learn-to-test}

Statistical lore tells us \emph{not to use the same data for learning and testing.} We now discuss 
whether
it is indeed possible to use the same data for selecting a test statistic from a candidate set and conducting the selected test
\citep{fithian2014optimal}. The key to controllable Type-I errors is that we need to adjust the test threshold to account for the selection event.  
As before, let $Z_n$ denote the data we collected. Let $\candidates = \{\tau_i\}_{i\in \mathcal{I}}$ be a countable set of candidate test statistics that we evaluate on the data $Z_n$, and $\{t_\alpha^i\}_{i\in \mathcal{I}}$ the respective test thresholds.
Assume that $\{A_i\}_{i\in \mathcal{I}}$ are disjoint \emph{selection events} depending on $Z_n$ and that their outcomes determine which test statistic out of $T$ we apply.
Thus, 
all the tests and events are generally dependent 
via $Z_n$. To define a \emph{well-calibrated} test, we need to control the overall Type-I error, i.e., $P(\text{reject} | H_0)$. Using the law of total probability, we can rewrite this in terms of the selected tests
\begin{align}\label{eq:conditional-type-i}
    P(\text{reject} | H_0) = \sum_{i \in \mathcal{I}} P(\tau_i > t_\alpha^i | A_i, H_0) P(A_i | H_0).
\end{align}
To control the Type-I error $P(\text{reject} | H_0) \leq \alpha$, it thus suffices to control $P(\tau_i > t_\alpha^i | A_i, H_0) \leq \alpha$ for each $i\in\mathcal{I}$, i.e., the test thresholds need to take into account the conditioning on the selection event $A_i$. 
A \emph{naive} approach would wrongly calibrate the test such that $P(\tau_i > t_\alpha^i|H_0) \leq \alpha$, not accounting for the selection $A_i$ and thus 
would result
in an uncontrollable Type-I error. 
On the other hand, this reasoning directly tells us why data splitting works. There $A_i$ is evaluated on a split of $Z_n$ that is independent of the split used to compute $\tau_i$ and hence $P(\tau_i > t_\alpha^i | A_i, H_0) = P(\tau_i > t_\alpha^i| H_0)$. 

\vspace{-5pt}
\paragraph{Selecting tests with high power.}
Our objective  
in selecting
the test statistic is to maximize the power of the selected test.
To this end, 
we start from $d \in\mathbb{N}$ different \emph{base functions} 
$h_1,...,h_d$.
Based on observed data ${Z}_n = \{{z}_1, \ldots, {z}_n\} \sim P^n$, we can compute $d$ \emph{base} test statistics 
${\tau}_u :=\tau_u({Z}_n) = \frac{1}{n}\sum_{i=1}^n h_u({z}_i)$ 
for $u\in [d]$. Let $\bm\tau := (\tau_1, \ldots, \tau_d)^\top$ and $\bm\mu:= \expect{}{\bm h(Z)}$, where $\mathbf{h}(Z) = (h_1(Z),\ldots,h_d(Z))^\top$. 
Asymptotically, we have $\sqrt{n}(\bm\tau -\bm \mu) \overset{d}{\rightarrow} \mathcal{N}(\bm 0, \Sigma)$, with the variance of the asymptotic distribution given by $\Sigma = \text{Cov}[\bm h(Z)]$.\footnote{ 
In practice, we work with an estimate $\hat{\Sigma}$ of the covariance obtained from $Z_n$, which is justified since $\sqrt{n}\hat{\Sigma}^{-\frac{1}{2}}(\bm\tau -\bm \mu) \overset{d}{\rightarrow} \mathcal{N}(\bm 0, \mathop{I}_d)$ for  consistent estimates of the covariance.}
Now, for any $\bm\beta \in \mathbb{R}^d\setminus\{\bm 0\}$ that is independent of $\bm\tau$, the normalized test statistic $\tau_{\bm\beta} := \frac{\bm{\beta}^\top \bm\tau}{(\bm{\beta}^\top\Sigma \bm{\beta})^\frac{1}{2}}$ is asymptotically normal, i.e., $\sqrt{n}\left(\tau_{\bm\beta} - \frac{\bm{\beta}^\top \bm\mu}{(\bm{\beta}^\top\Sigma \bm{\beta})^\frac{1}{2}}\right) \overset{d}{\rightarrow} \mathcal{N}(0, 1)$. Following our considerations of Section \ref{sec:preliminaries}, the test with the highest power is defined by
\begin{align}\label{eq:beta_inf}
    \bm\beta^\infty :&= \argmax_{\|\bm\beta\|=1} \frac{\bm{\beta}^\top \bm\mu}{(\bm{\beta}^\top\Sigma \bm{\beta})^\frac{1}{2}} =\frac{\Sigma^{-1} \bm\mu}{\|\Sigma^{-1} \bm\mu\|},
\end{align}
where the constraint $\|\bm\beta\| =1$ is to ensure that the solution is unique,  since the objective of the maximization is a homogeneous function of order $0$ in $\bm\beta$.
The explicit form of $\bm\beta^\infty$ is proven in Appendix \ref{sec:proof_beta_inf}. 
Obviously, in practice, $\bm\mu$ is not known, so we use an estimate of $\bm\mu$ to select $\bm\beta$. The standard strategy to do so is to split the sample $Z_n$ into two independent sets and estimate $\bm\tau_\text{tr}$ and $\bm\tau_\text{te}$, i.e., two independent training and test realizations \citep{Gretton2012optimal,JitSzaChwGre2016, SceVar2019, liu2020learning}. One can then choose a suitable $\bm\beta$ by using $\bm\tau_\text{tr}$ as a proxy for $\bm\mu$. Then one tests with this $\bm\beta$ and $\bm\tau_\text{te}$. 
However, to our knowledge, there exists no principled way to decide in which
proportion to split the data, which will generally influence the power, as  shown in our experimental results in Section \ref{sec:experiments}. 

Our approach to maximizing the utility of the observed dataset is to use it for
both learning and testing. To do so, we have to derive an adjustment to the distribution of the
statistic under the null, in the spirit of the selective hypothesis testing described above. We will consider three different candidate sets $\candidates$ of test statistics, which are all constructed from the base test statistics $\bm\tau$. 
To do so, we will work with the asymptotic distribution of $\bm\tau$ under the null. To keep the notation concise, we include the $\sqrt{n}$ dependence into $\bm\tau$. Thus, we will assume  $\bm\tau \sim \mathcal{N}(\bm 0, \Sigma)$, where $\Sigma$ is known and strictly positive. 
We provide the generalization to singular covariance in Appendix \ref{app:singular}.

To select the test statistics, we maximize the SNR $\tau_{\bm\beta} = {\bm{\beta}^\top \bm\tau} / {(\bm{\beta}^\top\Sigma \bm{\beta})^\frac{1}{2}}$ and thus the test power over three different sets of candidate test statistics: \begin{inparaenum}
    \item $\candidates_\text{base} = \left\{\tau_{\bm\beta} \, | \, \bm\beta \in \{e_1,\dots, e_d\}\right\}$, i.e., we directly select from the base test statistics, 
    \item $\candidates_\text{Wald} = \left\{\tau_{\bm\beta} \, | \, \|\bm\beta \| =1\right\}$, where we allow for arbitrary linear combinations,
    \item $\candidates_\text{\name} = \left\{\tau_{\bm\beta} \, | \, \Sigma \bm\beta \geq \bm 0, \|\Sigma \bm\beta \| = 1 \right\}$, where we constrain the allowed values 
    to increase the power (see below).
\end{inparaenum}
The rule for selecting the test statistic from these sets is simply to select the one with the highest value. To design selective hypothesis tests, we  need to derive suitable selection events and the distribution of the maximum test statistic conditioned on its selection.

\subsection{Selection from a finite candidate set}\label{sec:discrete}
We start with $\candidates_\text{base} = \left\{\tau_{\bm\beta} \, | \, \bm\beta \in \{e_1,\dots, e_d\}\right\}$ and use the test statistic $\tau_\text{base} = \max_{\tau \in \candidates_\text{base}} \tau$. Since the selection is from a countable set and the selected statistic is a projection of $\bm\tau$, we can use the polyhedral lemma of \citet{Lee2016} to derive the conditional distributions. 
Therefore, we denote $u^* = {\argmax}_{u \in [d]} \frac{\tau_u}{\sigma_u}$, with $\sigma_u := (\Sigma_{uu})^\frac{1}{2}$,
and obtain
$\tau_\text{base} = \frac{\tau_{u^*}}{\sigma_{u^*}}$.  
The following corollary characterizes the conditional distribution. The proof is given in Appendix \ref{app:proof_cor_discrete}.
\begin{corollary}\label{cor:selection_discrete}
Let $\bm\tau \sim \mathcal{N}(\bm\mu, \Sigma)$, $\bm z := \bm \tau - \frac{ \Sigma e_{u^*} \tau_{u^*} }{\sigma^2_{u^*}}$, $\mathcal{V}^-(\hat{\bm z}) = \max_{j \in [d], j\neq u^*}, 
    \frac{\sigma_{u^*} \hat{z}_j}{\sigma_u^* \sigma_j - \Sigma_{u^* j}}$, and $\text{TN}(\mu, \sigma^2, a,b)$ denote a  normal distribution with mean $\mu$ and variance $\sigma^2$ truncated at $a$ and $b$. 
    Then the following statement holds:
\begin{align}\label{eq:asymp-finite}
    \left[\frac{\tau_{u^*}}{\sigma_{u^*}} \left| u^* = \argmax_{u \in [d]} \frac{\tau_u}{\sigma_u}, \bm z = \hat{\bm z}\right.\right] \overset{d}{=} 
\text{TN}\left(\frac{\mu_{u^*}}{\sigma_{u^*}}, 1, \mathcal{V}^- (\hat{\bm{z}}), \mathcal{V}^+ = \infty \right),
\end{align}
\end{corollary}

This scenario arises, for example, in kernel-based tests when the kernel parameters are chosen from a grid of predefined values \citep{gretton2012kernel,Gretton2012optimal}.
Corollary \ref{cor:selection_discrete} allows us to test using the same set of data that was used to select the test statistic, by providing the corrected asymptotic distribution \eqref{eq:asymp-finite}.
The only downside is its dependence on the parameter grid.
To overcome this limitation, several works have proposed to optimize for the parameters directly \citep{Gretton2012optimal,JitKanSanHaySch2018,JitXuSzaFukGre2017,JitSzaGre2017,JitSzaChwGre2016}.
Unfortunately, we cannot apply Corollary \ref{cor:selection_discrete} directly to this scenario.

\subsection{Learning from an uncountable candidate set}\label{sec:continuous_general}

To allow for more flexible tests, in the following we consider the candidate sets $\candidates_\text{Wald}$ and $\candidates_\name$ that contain uncountably many tests. 
For these sets, we cannot directly use \eqref{eq:conditional-type-i} to derive conditional tests, since the probability of selecting some given tests is 0.
However, we show that it is possible 
in
both cases to rewrite the test statistic such that we can build conditional tests based on \eqref{eq:conditional-type-i}.
First, for $T_\text{Wald}$,
we rewrite the entire test statistic including the maximization in closed form. Second, for $T_{\name}$ we derive suitable measurable selection events that allow us to rewrite the conditional test statistic in closed form and derive their distributions in Theorem \ref{thm:continuous}.
\vspace{-5pt}
\paragraph{Wald Test.}
We first allow for arbitrary linear combinations of the base test statistics $\bm\tau$. Therefore, define $\candidates_\text{Wald} = \left\{\tau_{\bm\beta} \, | \, \|\bm\beta \| =1\right\}$ and $\tau_\text{Wald} := \max_{\tau\in \candidates_\text{Wald}} \tau$. We denote the optimal $\bm\beta$ for this set as 
$
    \bm\beta_\text{Wald} := {\argmax}_{\|\bm\beta\|=1} \frac{\bm{\beta}^\top \bm\tau}{(\bm{\beta}^\top\Sigma \bm{\beta})^\frac{1}{2}}.
$
This optimization problem is the same as in \eqref{eq:beta_inf}, hence
$
    \bm\beta_\text{Wald} = \frac{\Sigma^{-1} \bm\tau}{\|\Sigma^{-1} \bm\tau\|},
$
and we can rewrite the "Wald" test statistic as
$
    \tau_\text{Wald} = \frac{\bm{\beta}_\text{Wald}^\top \bm\tau}{(\bm{\beta}_\text{Wald}^\top\Sigma \bm{\beta}_\text{Wald})^\frac{1}{2}} = (\bm\tau^\top\Sigma^{-1} \bm\tau)^\frac{1}{2} = \|\Sigma^{-\frac{1}{2}}\bm\tau \|.
$ 
Note that $\candidates_\text{Wald}$ contains uncountably many tests. However, instead of deriving individual conditional distributions, we can directly derive the distribution of the maximized test statistic, since $\tau_\text{Wald}$ can be written in closed form. In fact, under the null, we have $\Sigma^{-\frac{1}{2}}\bm\tau \sim \mathcal{N}(\bm 0, \mathop{I}_d)$ and 
$\tau_\text{Wald}$
follows a chi distribution with $d$ degrees of freedom. Surprisingly, the presented approach results in the classic Wald test statistic 
\citep{Wald1943}, which originally was defined directly in closed form.

\vspace{-5pt}
\paragraph{One-sided test (\name).} 
The original Wald test was defined to optimally test $H_0: \bm\mu = \bm 0$ against the alternative $H_A: \bm\mu \neq  \bm 0$ \citep{Wald1943}. 
Thus, it 
ignores
the fact
that we only test against the "one-sided" alternative $\bm\mu \geq \bm 0$, which suffices since we consider linear-time estimates of the squared MMD as test statistics 
and their population values are non-negative.
Multiplying \eqref{eq:beta_inf} with $\Sigma$ yields 
$
    \Sigma \bm\beta^\infty = \frac{\bm \mu}{\|\Sigma^{-1}\bm\mu\|}.
    $
    Using $\bm\mu \geq \bm 0$, we find $ \Sigma \bm\beta^\infty \geq \bm 0$.
Thus, we have prior knowledge over the asymptotically optimal combination $\bm\beta^\infty$. 
To incorporate this, we a priori constrain the considered values of $\bm\beta$ 
by 
the condition $\Sigma \bm\beta \geq \bm 0$. 
Thus we define $\candidates_\text{\name} = \left\{\tau_{\bm\beta} \, | \, \Sigma \bm\beta \geq \bm 0, \|\Sigma \bm\beta \| = 1 \right\}$, where the norm constraint  $\|\Sigma \bm\beta\|=1$ is added to make the maximum unique. We suggest using the test statistic 
$\tau_\name := \max_{\tau\in \candidates_\name} \tau$. 
Before we derive suitable conditional distributions for this test statistic, we  rewrite it in a \emph{canonical form}.

\begin{remark}\label{rmk:generalization}
Define $\bm\alpha := \Sigma \bm\beta$, $\bm\rho := \Sigma^{-1}\bm\tau$, and $\Sigma' := \Sigma^{-1}\Sigma\Sigma^{-1} = \Sigma^{-1}.$ This implies $\bm\rho \sim \mathcal{N}(\bm 0, \Sigma')$ and
$
    \tau_\text{\name} := \max_{\|\Sigma\bm\beta\|  =1 ,\Sigma\bm\beta \geq \bm0} \frac{\bm{\beta}^\top \bm\tau}{(\bm{\beta}^\top\Sigma \bm{\beta})^\frac{1}{2}} = \max_{\|\bm\alpha\| =1,\bm\alpha \geq \bm 0} \frac{\bm{\alpha}^\top \bm\rho}{(\bm{\alpha}^\top\Sigma' \bm{\alpha})^\frac{1}{2}}.
$
\end{remark}

Thus in the following, we focus on the canonical form, where the constraints are simply positivity constraints. For ease of notation, we stick with $\bm\tau$ and $\Sigma$ instead of $\bm\rho$ and $\Sigma'$.
We will thus analyze the distribution of
\begin{align}\label{eq:objective}
\max_{\|\bm\beta\|  =1 ,\bm\beta \geq \bm 0} \frac{\bm{\beta}^\top \bm\tau}{(\bm{\beta}^\top\Sigma \bm{\beta})^\frac{1}{2}} = \frac{\bm\beta^{*\top}\bm\tau}{(\bm\beta^{*\top}\Sigma \bm\beta^{*})^\frac{1}{2}}, 
\end{align}
where $\bm \beta^*(\bm\tau) := {\argmax}_{\|\bm\beta\|  =1 ,\bm\beta \geq \bm 0} \frac{\bm{\beta}^\top \bm\tau}{(\bm{\beta}^\top\Sigma \bm{\beta})^\frac{1}{2}}$. 
We emphasize that $\bm\beta^*(\bm\tau)$ is a random variable that is determined by $\bm\tau$. For conciseness, however, we will use $\bm\beta^*$ and keep the dependency implicit. 
We find the solution of \eqref{eq:objective} by solving an equivalent convex optimization problem, which we provide in Appendix \ref{sec:optimization}.
We need to characterize the distribution of \eqref{eq:objective} under the null hypothesis, i.e., $\bm\tau \sim \mathcal{N}(\bm 0, \Sigma)$. 
Since we are not able to give an analytic form for $\bm\beta^*$, it is hard to directly compute the distribution of $\tau_\name$ as we did for the Wald test. 
In Section \ref{sec:discrete} we were able to work around this by deriving the distribution conditioned on the selection of $\bm\beta^*$. 
In the present case, however, there are uncountably many values that $\bm\beta^*$ can take, so for some the probability is zero. Hence, the reasoning of \eqref{eq:conditional-type-i} does not apply and we cannot use the PSI framework of \citet{Lee2016}. 

Our approach to solving this is the following.
Instead of directly conditioning on the explicit value of $\bm\beta^*$, we  condition on the \emph{active set}. 
For a given $\bm\beta^*$, we define the active set as $\mathcal{U} := \{u \,|\, \beta^*_u \neq 0\}\subseteq [d]$. Note that the active set is a function of $\bm\tau$, defined via \eqref{eq:objective}. 
In Theorem \ref{thm:continuous} we show that given the active set, we can derive a closed-form expression for $\bm\beta^*$, and we can characterize the distribution of the test statistic conditioned on the active set. Figure \ref{fig:proof_sketch} depicts the intuition behind Theorem \ref{thm:continuous} and Appendix \ref{app:proof_continuous} contains the full proof.
In the following, let $\chi_l$ denote a chi distribution with $l$ degrees of freedom and $\text{TN}\left(0,1,a,\infty\right)$ denote the distribution of a standard normal RV truncated from below at $a$, i.e., with CDF 
$
        F^{a}(x) = \frac{\Phi(x) - \Phi(a)}{1 - \Phi(a)}.
$
\begin{thm}\label{thm:continuous}
Let $\bm\tau \sim \mathcal{N}(\bm 0, \Sigma)$ be a normal RV in $\mathbb{R}^d$ with positive definite covariance matrix $\Sigma$. Let $\bm\beta^*$ be defined as in \eqref{eq:objective}, $\mathcal{U} := \{u \,|\, \beta^*_u \neq 0\}$, $l := |\mathcal{U}|$, $\bm z := \left({\mathop{I}}_d - \frac{ \Sigma\bm\beta^* \bm\beta^{*\top} }{\bm{\beta}^{*\top}\Sigma \bm{\beta^*}}\right) \bm\tau$, and $\mathcal{V}^-$ as in Corollary \ref{cor:selection_discrete}. 
Then, the following statements hold.
\begin{enumerate}
    \item [1.)] If $l=1$:$
        \qquad \left[
        {\displaystyle\max_{\|\bm\beta\|  =1 ,\bm\beta \geq \bm 0}} \frac{\bm{\beta}^\top \bm\tau}{(\bm{\beta}^\top\Sigma \bm{\beta})^\frac{1}{2}}
        \,\Big|\, \mathcal{U}, \bm z = \hat{\bm z} \right] \overset{d}{=} 
        \text{TN}\left(0,1,\mathcal{V}^- (\hat{\bm{z}}), \infty \right).
        $
    \item[2.)] If $l \geq 2$:$
    \qquad \left[
    \ {\displaystyle\max_{\|\bm\beta\|  =1 ,\bm\beta \geq \bm 0}} \frac{\bm{\beta}^\top \bm\tau}{(\bm{\beta}^\top\Sigma \bm{\beta})^\frac{1}{2}}
    \,\Big|\, \mathcal{U}\right] \overset{d}{=} \chi_l.
    $
\end{enumerate}
\end{thm}

\begin{figure}[t]
\begin{tikzpicture}
[    scale=1.4,
    axis/.style={->, >=stealth'},
    important line/.style={thick},
    dashed line/.style={dashed, thick},
    pile/.style={thick, ->, >=stealth', shorten <=2pt, shorten
    >=2pt},
    every node/.style={color=black}
    ]
 \draw[axis] (-2,0) -- (2,0) node(xline)[right]{$\tau_1$};
 \draw[axis] (0,-2) -- (0,2) node(yline)[right]{$\tau_2$};
 
\draw[dashed, color=blue, thick] (0,0) -- (-2,-2);
\draw[dotted, very thick](0,0) circle (1.);
\fill[pattern color=OliveGreen, pattern=north east lines] (0,0) rectangle (2,2);
\fill[pattern color=purple, pattern=horizontal lines] (0,0) rectangle (2,-2);
\fill[pattern color=purple, pattern=horizontal lines] (0,0) -- (0,-2) -- (-2,-2);
\fill[pattern color=YellowOrange, pattern=vertical lines] (0,0) rectangle (-2,2);
\fill[pattern color=YellowOrange, pattern=vertical lines] (0,0) -- (-2,0) -- (-2,-2);
\node[draw, fill=white] at (1.3,1.4) {$\begin{aligned}
    &l=2\\
    &\bm\beta^* = \frac{\hat{\bm\tau}}{\|\hat{\bm\tau}\|}
\end{aligned}$};
\node[draw, fill=white] at (1.2,-1.5) {$\begin{aligned}
    &l=1\\
    &\bm\beta^* = e_1
\end{aligned}$};
\node[draw, fill=white] at (-1.2,1.5) {$\begin{aligned}
    &l=1\\
    &\bm\beta^* = e_2
\end{aligned}$};
\end{tikzpicture}
\hfill
%
%
\begin{tikzpicture}
[    scale=1.4,
    axis/.style={->, >=stealth'},
    important line/.style={thick},
    dashed line/.style={dashed, thick},
    pile/.style={thick, ->, >=stealth', shorten <=2pt, shorten
    >=2pt},
    every node/.style={color=black}
    ]
 \draw[axis] (-2,0) -- (2,0) node(xline)[right]{$\tau_1$};
 \draw[axis] (0,-2) -- (0,2) node(yline)[right]{$\tau_2$};
 
\draw[dashed, color=blue, thick] (0,0) -- (-2,-2);
\draw[dotted, very thick](0,0) circle (1.);
\fill[pattern color=OliveGreen, pattern=north east lines, opacity=.3] (0,0) rectangle (2,2);
\fill[pattern color=purple, pattern=horizontal lines, opacity=.3] (0,0) rectangle (2,-2);
\fill[pattern color=purple, pattern=horizontal lines, opacity=.3] (0,0) -- (0,-2) -- (-2,-2);
\fill[pattern color=YellowOrange, pattern=vertical lines, opacity=.3] (0,0) rectangle (-2,2);
\fill[pattern color=YellowOrange, pattern=vertical lines, opacity=.3] (0,0) -- (-2,0) -- (-2,-2);
\draw[very thick, ->] (0,0) -- (1.2,-1.3) node[above right]{$\hat{\bm\tau}$};
\draw[very thick, ->] (0,0) -- (1,0) node[above left]{$\bm\beta^*$};
\draw[very thick, ->] (0,0) -- (0,-1.3) node[above left]{$\hat{\bm{z}}$};
\draw[dashed] (-1.3, -1.3) -- (-1.3, 0);
\filldraw(-1.3,0) node[above]{$\mathcal{V}^-$};
\draw[dashed] (-1.3, -1.3) -- (2, -1.3);
\draw[decoration={brace,mirror,raise=1pt},decorate, thick]
  (0,-1.3) -- node[below=2pt] {$\frac{\bm\beta^{*\top}\hat{\bm\tau}}{\left(\bm\beta^{*\top}\Sigma \bm\beta^*\right)^\frac{1}{2}}$} (1.2,-1.3);
\end{tikzpicture}
    \caption{Geometric interpretation of Theorem \ref{thm:continuous} for $d=2$ and unit covariance $\Sigma = \mathop{I}$ (denoted by the black dotted unit-circle). \textbf{Left:} If $\hat{\bm\tau}$ is in the positive quadrant (green), the constraints of the optimization are not active and the optimal direction is the same as for the Wald test, hence the distribution of the test statistic follows $\chi_2$. When $\hat{\bm\tau}$ is in the orange or purple zone, one of the constraints is active and $\bm\beta^*$ is a canonical unit-vector.
    \textbf{Right:} If $l=1$, for example when only the first direction is active, we additionally condition on $\bm{z}=\hat{\bm{z}}$, which is independent of the value of $\bm\beta^{*\top}\bm\tau$ since $\bm{z}$ is orthogonal to $\bm\beta^*$. For the observed value $\hat{\bm{z}}$, we only select $\bm\beta^* = e_1$ if $\bm\beta^{*\top}\bm\tau \geq \mathcal{V}^-$. If this was not the case, then $\bm\tau$ would lie in the orange/vertically lined region and we would select $\bm\beta^* = e_2$. This explains the truncated behavior and is in analogy to the results of \citet{Lee2016}.}
    \label{fig:proof_sketch}
\end{figure}
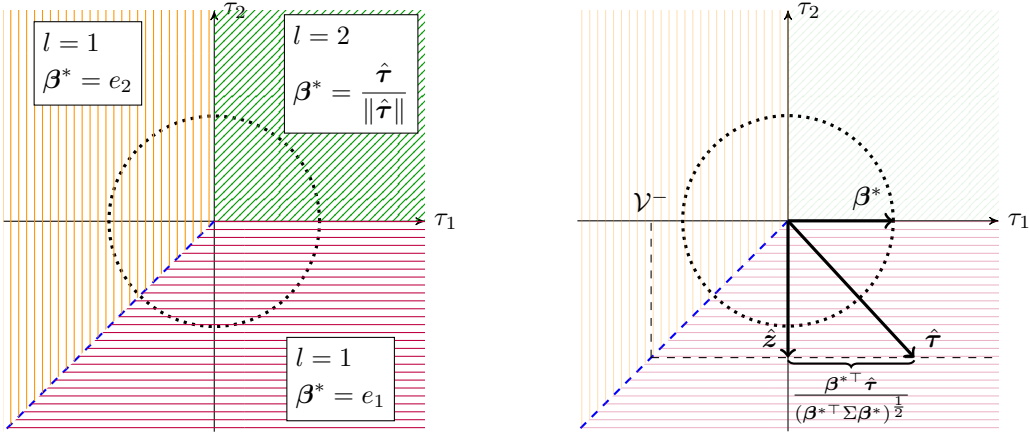

With Theorem \ref{thm:continuous} and Remark \ref{rmk:generalization}, we are able to define conditional hypothesis tests with the  test statistic $\tau_\name$. 
First, we transform our observation $\hat{\bm\tau}$ according to Remark \ref{rmk:generalization} to obtain it in canonical form, i.e., $\hat{\bm\tau} \rightarrow \Sigma^{-1} \hat{\bm\tau}$ and $\Sigma \rightarrow \Sigma^{-1}$. Then we solve the optimization problem \eqref{eq:objective} to find $\bm\beta^*$. Next, we define the active set $\mathcal{U}$, by checking which entries of $\bm\beta^*$ are non-zero. Theorem \ref{thm:continuous} characterizes the distribution $\tau_\name$ conditioned on the selection.
We can then define a test threshold $t_\alpha$
that accounts for the selection of $\mathcal{U}$, i.e.,
\begin{align}
    t_\alpha = \begin{cases}\Phi^{-1} \left((1-\alpha)(1-\Phi(\mathcal{V}^-)) + \Phi(\mathcal{V}^-)  \right) & \text{ if } |\mathcal{U}| = 1, \\
    \Phi^{-1}_{\chi_l}(1-\alpha) & \text{ if } |\mathcal{U}| = l\geq 2,
    \end{cases}
\end{align}
with $\Phi^{-1}_{\chi_l}$ being the inverse CDF of a chi distribution with $l$ degrees of freedom, which we can evaluate using standard libraries, e.g., \citet{scipy}. 
We can then reject the null, if the observed value of the optimized test statistic exceeds this threshold, i.e., $\hat{\tau}_\name > t_\alpha$.
We summarize the entire approach in Algorithm \ref{alg:continuous}.

\section{Related work}

Our work is best positioned in the context of modern statistical tests with
tunable hyperparameters. \citet{Gretton2012optimal} were the
first to propose a kernel two-sample test that optimizes the kernel
hyperparameters by maximizing the test power.
This influential work has led to
further development of 
optimized kernel-based tests
\citep{sutherland2016generative, JitKanSanHaySch2018, JitSzaGre2017, JitXuSzaFukGre2017,
JitSzaChwGre2016,SceVar2019}. 
Since any universally consistent binary classifier can be used to construct a valid two-sample test \citep{Friedman03:Tests,BS09:KernelChoice}, 
\citet{kim2016classification,LopOqu2017} used classification accuracy as a proxy to train machine learning models for two-sample tests.
\citet{CaiGogJia2020, KirKhoKloLip2019} studied this further, and 
\citet{CheClo2019} proposed using the difference
of a trained deep network's expected logit values  
as the test statistic
for two-sample tests. 

All the aforementioned ``learn-then-test'' approaches optimize
hyperparameters (e.g., kernels, weights in a network) on a training set which
is split from the full dataset. While the null distribution becomes tractable
due to the independence between the optimized hyperparameters and the test
set, there is a potential reduction of test power because of a smaller test
set. This observation is the main motivation for our consideration of selective hypothesis tests, which allow the full dataset to be used for both training and
testing by correcting for the dependency, as we discuss in Section \ref{sec:learn-to-test}.

More broadly, properly assessing the strength of potential associations that have been previously learned from the data falls under an emerging subfield of statistics known as \emph{selective inference} \citep{Taylor15:SI}. 
A seminal work of \citet{Lee2016} proposed a
post-selection inference (PSI) framework to characterize the valid
distribution of a post-selection estimator where model selection is performed
by the Lasso \citep{Tibshirani96:Lasso}.
The PSI framework has been applied to kernel tests, albeit in different context, for selecting the most
informative features for supervised learning \citep{YamUmeFukTak2018,SliChaAzeVer2019},
selecting a subset of features that best discriminates two 
samples
\citep{yamada2018post}, as well as selecting a model with the best fit from a
list of candidate models \citep{LimYamSchJit2019}. 
All these applications of the PSI framework consider a finite candidate set.
Our Theorem \ref{thm:continuous} can be seen as an extension of the previously known results of \citet{Lee2016} to uncountable candidate sets. To our knowledge, our work is the first to explicitly maximize test power by using the same data for selecting and testing. 


Unfortunately, we cannot directly use our results to optimize tests based on complete U-statistics estimates of the MMD, which would be desirable since those estimates have lower variance than the \emph{linear} version we use. The difficulty arises since our method requires asymptotic normality under the null, which is not the case for complete U-statistics \citep{gretton2012kernel}. To circumvent this problem, 
\citet{yamada2018post} considered incomplete U-statistics  \citep{Janson1984} and \citet{ZarGreBla2013} used a Block estimate of the MMD. Under the null, these approaches either have approximately asymptotic normal distribution \citep{yamada2018post} or require a higher sample size to reach the asymptotic normality \citep{ZarGreBla2013}. In principle thus our approach is applicable with these methods if one is willed to assume asymptotic normality and to neglect the induced errors. Besides that, since the linear-time estimate has lowest computational cost, it should generally be used in the \emph{large-data, constraint-computation} regime \citep{Gretton2012optimal}. On the other hand one should consider the other approaches when the computational efforts are not the limiting factor.

Moreover, under the assumption that $\bm \tau \sim \mathcal{N}(\bm\mu, \Sigma)$, similar scenarios have previously been investigated in the traditional statistical literature, but the idea of data splitting is not considered there.  
In particular, our construction of $\tau_\text{Wald}$ turned out to coincide with the test statistic suggested in \citet{Wald1943}. 
The one-sided version $\tau_\name$ also has a twin named ``\emph{chi-bar-square}'' test previously considered in \citet{kudo1963multivariate}.
While their test statistic is constructed to be always non-negative, our $\tau_\name$ can be negative.
Furthermore, they derived the distribution of the test statistic by decomposing the distribution into $2^d$ selection events, which, however, ``\emph{may represent a quite difficult problem}'' \citep[p.~54]{Shapiro88}. 
Our work circumvents this difficulty by defining a conditional test, which does not require calculating any probability of the selection events. Another difference is that our approach only defines $2^d-1$ different active sets, by enforcing $\bm\beta\neq \bm 0$.
It is instructive to note that there exist other more complicate settings of ``learn-then-test'' scenarios in which the normality assumption may not hold \citep{LopOqu2017,KirKhoKloLip2019,CheClo2019,CaiGogJia2020}. 
Extending our work towards these scenarios remains an open, yet promising problem to consider.

\section{Experiments}
\label{sec:experiments}
\begin{figure}[t]
\centering
\includegraphics[width=\linewidth]{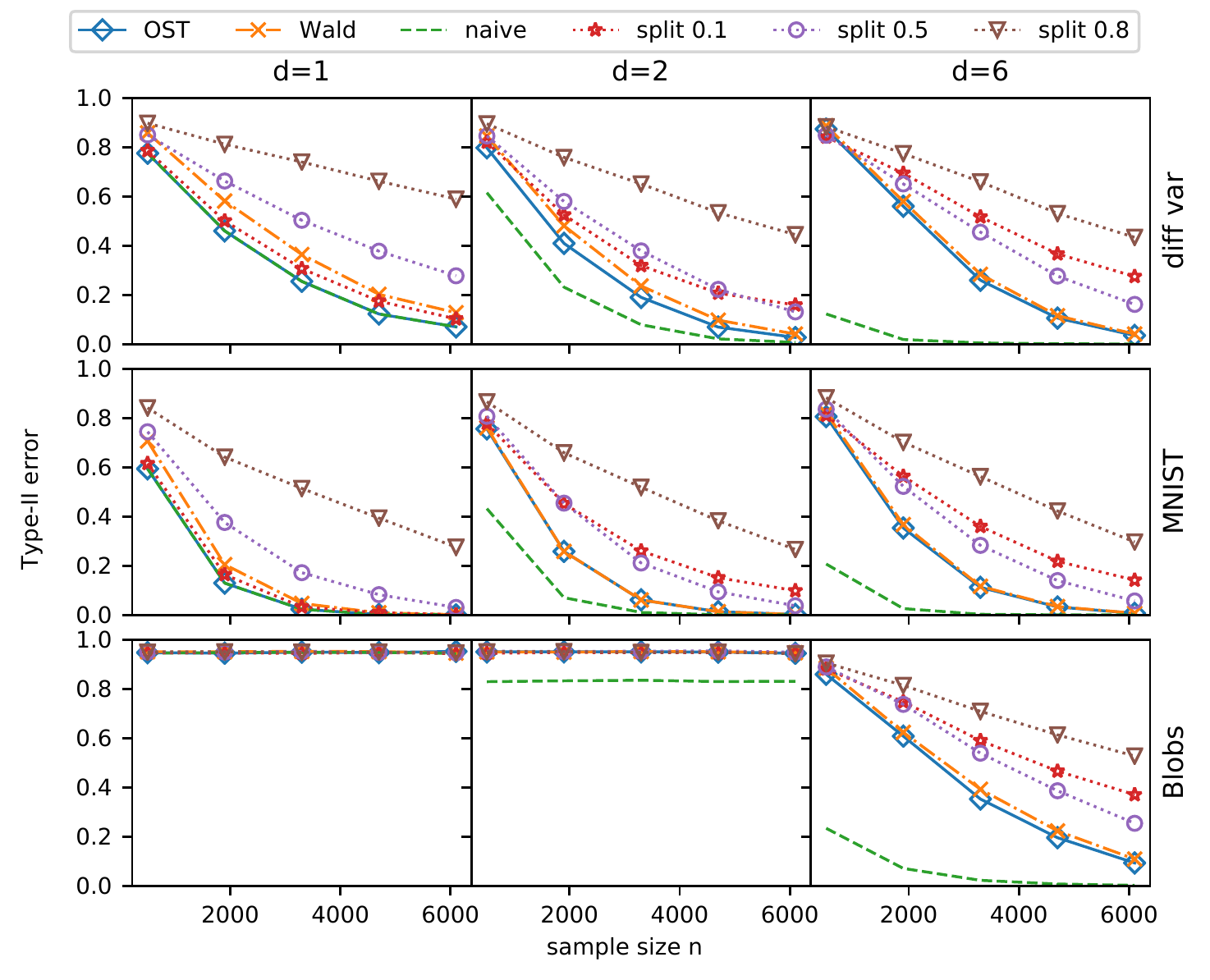}
\caption{Type-II errors obtained from different experiments. The rows (columns) correspond to different datasets (sets of base kernels). For all considered cases, \textsc{\name} outperforms all the (well-calibrated) competing methods, i.e., \textsc{Split} and \textsc{Wald}. 
}
\label{fig:typeII}
\end{figure}
We demonstrate the advantages of {\name} over data-splitting approaches and the Wald test with kernel two-sample testing problems as described in Section \ref{sec:preliminaries}. For an extensive description of the experiments we refer to Appendix \ref{app:Experimental_Details}. We consider three different datasets with different input dimensions $p$.
\begin{inparaenum}
    \item \texttt{DIFF VAR} ($p=1$): $P = \mathcal{N}(0,1)$ and $Q=\mathcal{N}(0,1.5)$. 
    \item \texttt{MNIST} ($p=49$): We consider downsampled 7x7 images of the MNIST dataset \citep{lecun2010mnist}, where $P$ contains all the digits and $Q$ only uneven digits.  
    \item \texttt{Blobs} ($p=2$): A mixture of anisotropic Gaussians where the covariance matrix of the Gaussians have different orientations for $P$ and $Q$.
\end{inparaenum}
We denote 
by
$k_\text{lin}$ the linear kernel, and $k_\sigma$ the Gaussian kernel with bandwidth  $\sigma$. For each dataset we consider three different base sets of kernels $\mathcal{K}$ and choose $\tilde{\sigma}$ with the median heuristic:
\begin{inparaenum}[(a)]
    \item $d=1$: $\mathcal{K}= [k_{\tilde{\sigma}}]$,
    \item $d=2$: $\mathcal{K}= [k_{\tilde{\sigma}}, k_\text{lin}]$,
    \item $d=6$: $\mathcal{K}= [k_{0.25\tilde{\sigma}}, k_{0.5\tilde{\sigma}}, k_{\tilde{\sigma}}, k_{2\tilde{\sigma}}, k_{4\tilde{\sigma}}, k_\text{lin}]$.
\end{inparaenum} 
From the base set of kernels we estimate the base set of test statistics using the linear-time MMD estimates.
We compare four different approaches:
\begin{inparaenum}[i)]
    \item \textsc{\name}, 
    \item \textsc{Wald}, 
    \item \textsc{split}:
    Data splitting similar to the approach in \citet{Gretton2012optimal}, but with the same constraints as {\name}. \textsc{split0.1} denotes that 10\% of the data are used for learning $\bm\beta^*$ and 90\% are used for testing,
    \item \textsc{naive}:
    Similar to splitting but all the data is used for learning and testing without correcting for the dependency. The \textsc{naive} approach is not a well-calibrated test.
\end{inparaenum}
For all the setups we estimate the Type-II error for various sample sizes at a level $\alpha = 0.05$. Error rates are estimated over 5000 independent trials and the results are shown in Figure \ref{fig:typeII}. 
In Appendix \ref{app:type_I}, we also investigate the Type-I error and show that all methods except for \textsc{naive} correctly control the Type-I error at a rate $\alpha$.
Note that all of the methods scale with $\mathcal{O}(n)$ and the difference in computational cost are negligible.

The experimental results in Figure \ref{fig:typeII} support the main claims of this paper. First, comparing $\textsc{\name}$ with \textsc{split}, we conclude that using all the data in an integrated approach is always better (or equally good) than any data splitting approach. Second, comparing $\textsc{\name}$ to $\textsc{Wald}$, we conclude that adding a priori information ($\bm\mu\geq \bm 0$) to reduce the class of considered tests in a sensible way leads to higher (or equally high) test power. 
Another interesting observation is in the results of the data-splitting approach. Looking at the \textsc{diff var} experiment, in the leftmost plot, we can see that the errors are 
monotonically increasing with
the portion of data used to select the test. Since there is only one test, the more data we use to select the test, the higher the error (less data remains for testing). 
In the middle plot, selection becomes important. Hence, we can see that the gap
in
performance between all data-splitting approach reduces. However, the order is still consistent with the previous plot.
Interestingly, in the rightmost plot, learning becomes even more important. Now, the order changes. If we use too 
little
data for learning the test (\textsc{split0.1}), the error is high. 
However, if we use too much data for learning the test (\textsc{split0.8}), the error will be high 
as well.
That is, there is a trade-off in how much data one should use for selecting the test, and 
for
conducting the test. 
The optimal proportion
depends on the problem and can thus in general not be determined a priori.  

In the Appendix \ref{app:discrete} we also compare $\tau_\text{base}$ to a selection of a base test via the data-splitting approach. Here, \textsc{split0.1} consistently performs better than the other split approaches, which is plausible, since the class of considered tests $T_\text{base}$ is quite small. \textsc{Split0.1} can even be better than $\tau_\text{base}$, see discussion in Appendix \ref{app:discrete}.

In Figure \ref{fig:laplaceII}, we additionally consider a constructed $1$-D dataset where the distributions share the first three moments and all uneven moments vanish (Figure \ref{fig:plot_uniform} in the appendix). We compare the results for different sets of $d \in [5]$ base kernels $\mathcal{K} = [k_\text{pol}^1, \dots, k_\text{pol}^d]$, where $k_\text{pol}^u(x,y) = (x\cdot y)^u$ denotes the homogeneous polynomial kernel of order $u$. 
By construction, $k_\text{pol}^u$ does not contain any information about the difference of $P$ and $Q$, for $u\neq 4$.
Thus, for $d\leq 3$ the well-calibrated methods have a Type-II error of $1-\alpha$. 
Only the \textsc{naive} approach already overfits to the noise. 
Adding the fourth order polynomial adds helpful information and all the methods improve performance. However, adding the fifth order, which again only contains noise, leads to an increased error rate. 
We interpret this as bias-variance tradeoff that should be considered in the construction of the base set $\mathcal{K}$.

In Appendix \ref{app:constraints} we compare how the constraints $\bm\beta \geq \bm 0$, as suggested in \citet{Gretton2012optimal}, work in comparison to the {\name}  approach. We find that while the constraints $\Sigma \bm\beta \geq \bm 0$ lead to consistently higher power than the Wald test, the simple positivity constraints can lead to both, better or worse power depending on the problem. 
We thus recommend using the 
{\name}.

\begin{figure}[t]
    \centering
{\begin{minipage}[t]{.47\linewidth}
\begin{algorithm}[H]
\caption{One-Sided Test (\name)}
  \begin{algorithmic}
      \INPUT $\Sigma$, $\hat{\bm\tau} = \sqrt{n} \widehat{\text{MMD}}^2(P,Q)$, $\alpha$
      \STATE $\hat{\bm\tau} = \Sigma^{-1}\hat{\bm\tau}$ \COMMENT{Apply Remark \ref{rmk:generalization}}
      \STATE $\Sigma = \Sigma^{-1}$ \COMMENT{Apply Remark \ref{rmk:generalization}}
      \STATE $\bm\beta^* = {\argmax}_{\|\bm\beta\|  =1 ,\bm\beta \geq \bm 0} \frac{\bm{\beta}^\top \hat{\bm\tau}}{(\bm{\beta}^\top\Sigma \bm{\beta})^\frac{1}{2}}$
      \STATE $\mathcal{U} = \{u|u \in [d], \beta^*_u >0\}$
      \STATE $\hat{\bm z} = \hat{\bm \tau} - \Sigma\bm\beta^* \frac{\bm\beta^{*\top} \hat{\bm\tau}}{\bm{\beta}^{*\top}\Sigma \bm{\beta^*}}$
      \STATE $ l = |\mathcal{U}|$
      \IF{$ l \geq 2$}
      \STATE $t_\alpha = \Phi^{-1}_{\chi_l}(1-\alpha)$
      \ENDIF
      \IF{$l = 1$}
      \STATE $\mathcal{V}^-
            =\max_{u\notin \mathcal{U}} \frac{\hat{z}_u (\bm{\beta}^{*\top}\Sigma \bm{\beta}^*)^\frac{1}{2}}{\Sigma_{uu}^\frac{1}{2}(\bm{\beta}^{*\top}\Sigma \bm{\beta}^*)^\frac{1}{2} - (\Sigma \bm\beta^*)_u}$
      \STATE $t_\alpha =  \Phi^{-1} \left((1-\alpha)(1-\Phi(\mathcal{V}^-)) + \Phi(\mathcal{V}^-)  \right)$    
      \ENDIF
      \IF{$t_\alpha < \frac{\bm\beta^{*\top}\hat{\bm\tau}}{(\bm{\beta}^{*\top}\Sigma \bm{\beta}^*)^\frac{1}{2}}$}
      \STATE Reject $H_0$
      \ENDIF
  \end{algorithmic}\label{alg:continuous}
\end{algorithm}
\end{minipage}
}
\hfill
\begin{minipage}[t]{.47\linewidth}
\begin{figure}[H]
\centering
    \includegraphics[width=.99\textwidth]{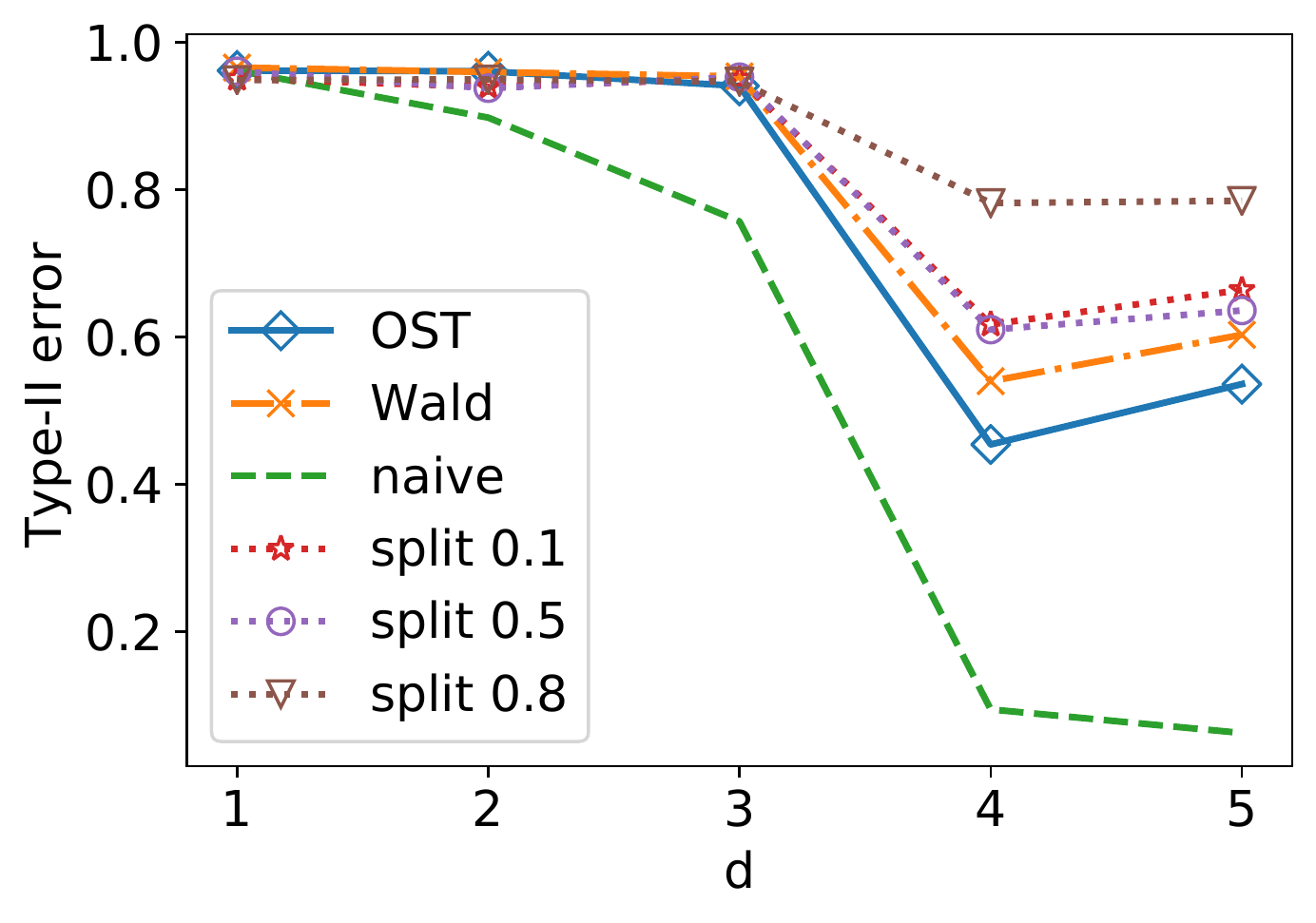}
    \caption{Type-II errors when the first $d$ polynomial kernels are used for a two-sample test with symmetric distributions with the equal covariance (Figure \ref{fig:plot_uniform} in the appendix). {\name} outperforms all the (well-calibrated) competitors. 
    }
    \label{fig:laplaceII}
\end{figure}

\end{minipage}
\end{figure}

\section{Conclusion}
Previous work used data splitting to exclude dependencies when optimizing a hypothesis test. 
This work is the first step towards using all the data for learning and testing. Our approach uses asymptotic joint normality of a predefined set of test statistics to derive the conditional null distributions in closed form.  
We investigated the example of kernel two-sample tests, where we use linear-time MMD estimates of multiple kernels as a base set of test statistics.
We experimentally verified that an integrated approach outperforms the existing data-splitting approach of \citet{Gretton2012optimal}. Thus data splitting, although theoretically easy to justify, does not efficiently use the data. 
Further, we experimentally showed that a one-sided test ({\name}), using prior information about the alternative hypothesis, leads to an increase in test power compared to the more general Wald test. 
Since the estimates of the base test statistics are linear in the sample size and the null distributions are derived analytically, the whole procedure is computationally cheap. 
However, it is an open question  whether and how this work can be generalized to problems where the class of candidate tests is not directly constructed from a base set of jointly normal test statistics. 

\section*{Broader impact}
 Hypothesis testing and valid inference after model selection are fundamental problems in statistics, which have recently attracted increasing attention also in machine learning. Kernel tests such as MMD are not only used for statistical testing, but also to design algorithms for deep learning and GANs \citep{LiSweZem2015,LiChaCheYan2017}. 
 The question of how to select the test statistic naturally arises in kernel-based tests because of the kernel choice problem. Our work shows that it is possible to overcome the need of (wasteful and often heuristic) data splitting when designing hypothesis tests with feasible null distribution. Since this comes without relevant increase in computational resources we expect the proposed method to replace the data splitting approach in applications that fit the framework considered in this work. Theorem \ref{thm:continuous} is also applicable beyond hypothesis testing and extends the previously known PSI framework proposed by \citet{Lee2016}. 
 
\begin{ack}
    The authors thank Arthur Gretton, Will Fithian, and Kenji Fukumizu for helpful discussion. JMK thanks Simon Buchholz for helpful discussions and pointing out a simplification of Lemma \ref{lemma:conditions_optimum}.
\end{ack}

\bibliography{refs}
\bibliographystyle{unsrtnat}

\newpage
\appendix


\section{Proof of Theorem \ref{thm:continuous}}\label{app:proof_continuous}

In this section we prove the main theorem. 
The outline of the proof is as follows: We first characterize the "selection event", i.e., we characterize under which conditions each active set $\mathcal{U}$ is selected. This is done with Lemmas \ref{lemma:conditions_optimum} and \ref{lemma:conditions_in_terms_of_z}. For the case $l=1$ we then show that the PSI framework of \citet{Lee2016} can be applied and we recover the result of Corollary \ref{cor:selection_discrete}. It is not surprising, that for the case $l=1$ the PSI framework works, since $\mathcal{U}$ corresponds to a single fixed $\bm\beta^*$ and the probability of selecting it is greater than $0$. For the case $l\geq 2$, we show, that the considered test statistic essentially takes the same form as the Wald test but only on the active dimensions. Thus it follows a $\chi_l$ distribution. This distribution does not change even if we explicitly condition on the selection of $\mathcal{U}$. This is because the randomness that determines which active set is selected is independent of the value of the selected test statistic. Before we start with the proof we collect some notation we introduce for the proof.
\textbf{Notation:}
\begin{itemize}
\item The objective of the optimization $f(\bm\beta) := \frac{\bm{\beta}^\top {\bm\tau}}{(\bm{\beta}^\top\Sigma \bm{\beta})^\frac{1}{2}}$.
    \item Projector onto the active subspace (leaving the dependency on $\mathcal{U}$ implicit):
\begin{align*}
    \proj := \sum_{u\in \mathcal{U}} e_u e_u^\top,
\end{align*}
where $e_u$ denotes the $u$-th Cartesian unit vector in $\mathbb{R}^d$.

\item $\bm z := \left({\mathop{I}}_d - \frac{ \Sigma\bm\beta^* \bm\beta^{*\top} }{\bm{\beta}^{*\top}\Sigma \bm{\beta^*}}\right) \bm \tau = \bm \tau - \Sigma\bm\beta^* \frac{\bm\beta^{*\top} \bm\tau}{\bm{\beta}^{*\top}\Sigma \bm{\beta^*}}$.
\item $\pseudo$ denotes the pseudoinverse of $\proj\Sigma\proj$.
\end{itemize}

As a first step, we need to characterize which values of $\bm\tau$ correspond to which active set $\mathcal{U}$. This is done with Lemma \ref{lemma:conditions_optimum}, which we prove separately in \ref{app:proof_lemma_conditions}. 
\begin{lemma}\label{lemma:conditions_optimum}
    Let $\mathcal{U} := \{u \,|\, \beta^*_u \neq 0\}$. Then, 
    \begin{align*}
        \bm \beta^* = \argmax_{\|\bm\beta\|  =1 ,\bm\beta \geq \bm 0} \frac{\bm{\beta}^\top {\bm\tau}}{(\bm{\beta}^\top\Sigma \bm{\beta})^\frac{1}{2}}
    \end{align*}
    if and only if all of the following conditions hold:
    \begin{enumerate}
        \item $\left. \frac{\partial}{\partial\beta_u} \frac{\bm{\beta}^\top {\bm\tau}}{(\bm{\beta}^\top\Sigma \bm{\beta})^\frac{1}{2}} \right|_{\bm\beta=\bm\beta^*} \begin{cases} \leq 0 \text{ if } u \notin \mathcal{U}\quad (a), \\
        = 0 \text{ if } u \in \mathcal{U} \quad (b),
        \end{cases}$
        \item $ \frac{\bm{\beta}^{*\top} {\bm\tau}}{(\bm{\beta}^{*\top}\Sigma \bm{\beta}^*)^\frac{1}{2}} \geq \frac{\tau_u }{\sqrt{\Sigma_{uu}}}, \qquad \forall u \notin \mathcal{U}$,
        \item $\bm\beta^*_u
        = 0 \quad \forall u \notin \mathcal{U} \quad(a),$\\
        $\bm\beta^*_u > 0 \quad \forall u \in \mathcal{U} \quad(b)$,\\
        $\|\bm\beta^*\|  =1 $ \quad (c).
    \end{enumerate}
\end{lemma}
Intuitively, Condition1(b) ensures that $\bm\beta^*$ is a local maximum of the objective function for the active dimensions.
Condition 1(a) ensures that if $u \notin \mathcal{U}$,
increasing $\beta^*_u$ does not improve the SNR.
Condition 2 is harder to
interpret, but is needed in cases where all entries of $\bm\tau$ are
negative.
Condition 3 enforces that $\bm\beta^*$ lies in the feasible set of \eqref{eq:objective}.

Note that $\bm\beta^{*\top}\bm\tau$ is essentially a one-dimensional RV.
    We define another random variable 
    \begin{align}\label{eq:definition_z_cont}
        \bm z := \left({\mathop{I}}_d - \frac{ \Sigma\bm\beta^* \bm\beta^{*\top} }{\bm{\beta}^{*\top}\Sigma \bm{\beta^*}}\right) \bm \tau = \bm \tau - \Sigma\bm\beta^* \frac{\bm\beta^{*\top} \bm\tau}{\bm{\beta}^{*\top}\Sigma \bm{\beta^*}}.
    \end{align}

    In Appendix \ref{app:gradient}, we show that $\bm z$ is closely related to the partial derivatives of the objective function and we have
    \begin{align}\label{eq:partial_in_term_of_z}
        \left. \frac{\partial}{\partial\beta_u} \frac{\bm{\beta}^\top {\bm\tau}}{(\bm{\beta}^\top\Sigma \bm{\beta})^\frac{1}{2}} \right|_{\bm\beta=\bm\beta^*} = \frac{\bm z}{\left(\bm{\beta}^{*\top}\Sigma \bm{\beta^*}\right)^\frac{1}{2}}.
    \end{align}
    We can then rewrite the conditions of Lemma \ref{lemma:conditions_optimum} as follows.
    \begin{lemma}\label{lemma:conditions_in_terms_of_z}
        The conditions of Lemma \ref{lemma:conditions_optimum} are equivalent to
        \begin{enumerate}
        \item $\begin{cases}
        z_u \leq 0 \qquad \forall u \notin \mathcal{U} \quad & (a),\\
        z_u = 0 \qquad \forall u \in \mathcal{U} & (b),
        \end{cases}$
        \item $\frac{\bm{\beta}^{*\top} {\bm\tau}}{(\bm{\beta}^{*\top}\Sigma \bm{\beta}^*)^\frac{1}{2}} \geq \mathcal{V}^-(\bm z)$, with
        
            $\mathcal{V}^-(\bm z)
            :=\max_{u\notin\mathcal{U}} \frac{z_u (\bm{\beta}^{*\top}\Sigma \bm{\beta}^*)^\frac{1}{2}}{\Sigma_{uu}^\frac{1}{2}(\bm{\beta}^{*\top}\Sigma \bm{\beta}^*)^\frac{1}{2} - (\Sigma \bm\beta^*)_u},$
        \item $\bm\beta^*_u
        = 0 \quad \forall u \notin \mathcal{U} \quad(a),$\\
        $\bm\beta^*_u > 0 \quad \forall u \in \mathcal{U} \quad(b)$,\\
        $\|\bm\beta^*\|  =1 $ \quad (c).
    \end{enumerate}
    \end{lemma}
    \begin{proof}[Proof of Lemma \ref{lemma:conditions_in_terms_of_z}]
        Condition 1 directly follows from \eqref{eq:partial_in_term_of_z}.   The second condition follows by inserting the definition of $\bm z$
        \begin{align*}
            &\frac{\bm{\beta}^{*\top} {\bm\tau}}{(\bm{\beta}^{*\top}\Sigma \bm{\beta}^*)^\frac{1}{2}} \geq \frac{\tau_u }{\sqrt{\Sigma_{uu}}}\\
            \Leftrightarrow & \frac{\bm{\beta}^{*\top} {\bm\tau}}{(\bm{\beta}^{*\top}\Sigma \bm{\beta}^*)^\frac{1}{2}} \geq \frac{z_u }{\sqrt{\Sigma_{uu}}} +  e_u^\top\Sigma\bm\beta^* \frac{\bm\beta^{*\top} \bm\tau}{\bm{\beta}^{*\top}\Sigma \bm{\beta^*} \sqrt{\Sigma_{uu}}}\\
            \Leftrightarrow & \frac{\bm{\beta}^{*\top} {\bm\tau}}{(\bm{\beta}^{*\top}\Sigma \bm{\beta}^*)^\frac{1}{2}} \left(1 -  \frac{e_u^\top\Sigma\bm\beta^*}{(\bm{\beta}^{*\top}\Sigma \bm{\beta^*})^\frac{1}{2} \sqrt{\Sigma_{uu}}}\right) \geq \frac{z_u }{\sqrt{\Sigma_{uu}}}\\
             \Leftrightarrow & \frac{\bm{\beta}^{*\top} {\bm\tau}}{(\bm{\beta}^{*\top}\Sigma \bm{\beta}^*)^\frac{1}{2}} \left((\bm{\beta}^{*\top}\Sigma \bm{\beta^*})^\frac{1}{2} \sqrt{\Sigma_{uu}} -  e_u^\top\Sigma\bm\beta^*\right) \geq z_u (\bm{\beta}^{*\top}\Sigma \bm{\beta^*})^\frac{1}{2}\\
             \Leftrightarrow & \frac{\bm{\beta}^{*\top} {\bm\tau}}{(\bm{\beta}^{*\top}\Sigma \bm{\beta}^*)^\frac{1}{2}}  \geq \frac{z_u (\bm{\beta}^{*\top}\Sigma \bm{\beta^*})^\frac{1}{2}}{\left((\bm{\beta}^{*\top}\Sigma \bm{\beta^*})^\frac{1}{2} \sqrt{\Sigma_{uu}} -  e_u^\top\Sigma\bm\beta^*\right)},
        \end{align*}
        where we used $\Sigma_{uu}^\frac{1}{2}(\bm{\beta}^{*\top}\Sigma \bm{\beta}^*)^\frac{1}{2} - (\Sigma \bm\beta^*)_u > 0$, which holds since $\Sigma$ is positive and we only consider $u$ such that $e_u \neq \bm\beta^*$.
    \end{proof}
    Note that $\mathcal{V}^-(\bm z)$ is always non-positive by Condition 1 and the positivity of $\Sigma$.
    With the above two lemmas we are able to prove Theorem \ref{thm:continuous}.
    
    \begin{proof}[Proof of Theorem \ref{thm:continuous}] 
    
    We prove the two cases $l=1$ and $l\geq 2$ separately.
    
1.): Let $u^* \in [d]$ such that $\mathcal{U} = \{u^*\}$.
    In this case, by Condition 3, $\bm\beta^* = e_{u^*}$.
    We shall now see how Lemma \ref{lemma:conditions_in_terms_of_z} constrains the distribution of $\tau_{u^*}$. 
    For Condition 1(b), we have $z_{u^*} = 0$ by the definition of $\bm z$. 
    So there only remain the constraints 1(a) and 2. Using the definition \eqref{eq:definition_z_cont} of $\bm z$, we can rewrite 1(a) as
    \begin{align*}
       \left(\left({\mathop{I}}_d - \Sigma e_{u^*} \frac{e_{u^*}^\top}{\Sigma_{u^*u^*}}\right) \bm\tau \right)_u \leq 0 \quad \forall u\notin \mathcal{U}
      \Longleftrightarrow 
      A^{[1(b)]} \bm\tau \leq 0,
    \end{align*}
    where $ A^{[1(b)]}$ is the matrix $\left({\mathop{I}}_d - \Sigma e_{u^*} \frac{e_{u^*}^\top}{\Sigma_{u^*u^*}}\right)$ and we used that its $u$-th row contains only zeros.
    Note that Condition 2 is the same as used in Section \ref{sec:discrete}. Thus we can define the matrix $A^{[2]}$ as we do in the proof of Corollary \ref{cor:selection_discrete}. 
    We have now all the remaining constraints as linear inequalities of $\bm\tau$ and thus we can find the conditional distribution by applying Theorem \ref{thm:poly}. 
    Defining $\bm\eta = \frac{e_{u^*}}{(\bm\beta^*T\Sigma\bm\beta^*)^\frac{1}{2}}$ and $\bm{c} := \Sigma \bm\eta \left(\bm{\eta}^\top \Sigma \bm{\eta}\right)^{-1}$, we get $ A^{[1(b)]} \bm c = \bm 0$. Note that whenever $(A\bm c)_j = 0$, the constraint does not change anything in Theorem \ref{thm:poly}. 
    Thus the result follows by using $A = A^{[2]}$ and application of Theorem \ref{thm:poly}. 
    
    An alternative proof can be done by noting that $\bm z$ is independent of $\frac{\bm{\beta}^{*\top} {\bm\tau}}{(\bm{\beta}^{*\top}\Sigma \bm{\beta}^*)^\frac{1}{2}}$ if we consider $\bm\beta^* = e_{u^*}$ as fixed. Thus, the fulfillment of Condition 1b) is independent of $\frac{\bm{\beta}^{*\top} {\bm\tau}}{(\bm{\beta}^{*\top}\Sigma \bm{\beta}^*)^\frac{1}{2}}$. Since the unconditional distribution of $\frac{\bm{\beta}^{*\top} {\bm\tau}}{(\bm{\beta}^{*\top}\Sigma \bm{\beta}^*)^\frac{1}{2}}$ follows a standard normal, adding Condition 2 results in a truncated normal.

    2.) Next, we consider the case $|\mathcal{U}| \geq 2$. Again we will be considering the conditions as stated in Lemma \ref{lemma:conditions_in_terms_of_z}. As we state in \eqref{eq:positive_if_l>2}, we have $\bm\beta^{*\top}\bm\tau \geq 0$ and thus Condition 2 is fulfilled, since $\mathcal{V}^-$ is always non-positive. Thus, we can neglect Condition 2. Our first step will be to find a closed form function $h_\mathcal{U}$ such that $\bm\beta^* = h_\mathcal{U}(\bm\tau)$ (this function will only hold true if $\mathcal{U}$ is actually the active set). Defining the projector onto the active subspace $\proj := \sum_{u\in \mathcal{U}} e_u e_u^\top$, by Condition 3(a) we have $\bm\beta^* = \proj \bm\beta^*$.
    Using \eqref{eq:definition_z_cont}, we can rewrite Condition 1(b) as
    \begin{align}
        &\proj \bm z= \bm 0 
        \,\overset{\eqref{eq:definition_z_cont}}{\Leftrightarrow}\,
        \proj \bm\tau = \proj \Sigma\bm\beta^* \frac{\bm\beta^{*\top} \bm\tau}{\bm{\beta}^{*\top}\Sigma \bm{\beta^*}}
        \,\overset{3(a)}{\Leftrightarrow}\,
        \proj \bm\tau = \proj \Sigma\proj \bm\beta^* \frac{\bm\beta^{*\top} \bm\tau}{\bm{\beta}^{*\top}\Sigma \bm{\beta^*}}\label{eq:system_for_beta_star} .
    \end{align}
    This defines a system of $l$ non-trivial equations and by Condition 3, $\bm\beta^*$ has $l$ free parameters. We define $\pseudo$ as the pseudoinverse of $\proj\Sigma\proj$.\footnote{For intuition, assume WLOG that $\mathcal{U} = \{1,\dots, l\}$. The pseudoinverse is then simply the inverse of the $l\times l$ blockmatrix padded with zeros.} 
    For the pseudoinverse it is easy to show $\pseudo = \proj\pseudo  = \pseudo \proj$. 
    Since $\Sigma$ has full rank, a possible solution of \eqref{eq:system_for_beta_star} necessarily has to be of the form $\bm\beta^* = c \cdot \pseudo \bm\tau$ for some $c \in \mathbb{R}$.
     Plugging this into \eqref{eq:system_for_beta_star}, we get $c = \frac{\bm\beta^{*\top}\Sigma\bm\beta^*}{\bm\beta^{*\top} {\bm \tau}}$. Using \eqref{eq:positive_if_l>2}  we get
    $0 \leq \frac{\bm{\beta}^{*\top} {\bm\tau}}{(\bm{\beta}^{*\top}\Sigma \bm{\beta}^*)^\frac{1}{2}}  = \frac{1}{c}.$
    Hence, $c \geq 0$. Using 
    $\|\bm\beta^*\| = 1$ we get $c =\frac{1}{\| \pseudo\bm \tau\|}$.
    Thus, given that the active set is $\mathcal{U}$, we found a closed-form solution for $\bm\beta^*$ as a function of $\bm \tau$, i.e., 
    \begin{align}\label{eq:deterministic_beta_U}
        {\bm\beta}^* = h_{\mathcal{U}}({\bm \tau}) :=  \frac{\pseudo {\bm \tau}}{\|\pseudo {\bm \tau}\|}.
    \end{align}
    Note that so far we did not use Condition 3(b), so this formula itself does not ensure the positivity of $\bm\beta^*$.
    
   Replacing $\bm\beta^*$ in the definition \eqref{eq:definition_z_cont} of $\bm z$ with its closed form, the constant $c$ cancels, and  we get 
    \begin{align*}
        \bm z = \bm \tau - \Sigma \pseudo \bm\tau.
    \end{align*}
     Note that $\pseudo \proj\Sigma\proj \pseudo= \pseudo$ and $(\Sigma \pseudo)_{uu'} =\delta_{uu'} $ if $u,u'\in \mathcal{U}$. This implies that $z_u = 0$ if $u\in \mathcal{U}$ and thus also $\bm z^\top \pseudo \bm\tau = 0$. 
     
    Let us now define $\tilde{X} := (\pseudo)^{\frac{1}{2}}\bm\tau$, resulting in $\tilde{X}_u = 0$ for all $u \notin \mathcal{U}$.
    Since $\tilde{X}$ and $\bm z$ are both linear transformations of $\bm\tau$ they are jointly normally distributed. In Appendix \ref{app:correlation_X,z} we show that $\tilde{X}$ and $\bm z$ are uncorrelated. This, together with the joint normality, implies that they are independent, i.e.,
    \begin{equation}\label{eq:independence_X.z}
        \tilde{X} \perp \!\!\! \perp \bm z.
    \end{equation}
    Further the non-zero coordinates of $\tilde{X}$ are jointly distributed according to a $l$-dimensional standard normal distribution. Hence, its euclidean norm follows a chi-distribution
    \begin{align}
        \|\tilde{X}\|_2 \sim \chi_l.
    \end{align}
    Let us summarize how we used all the conditions of Lemma \ref{lemma:conditions_in_terms_of_z} and finish the proof.
    We used 1(b), 3(a), and 3(c) to show $\eqref{eq:deterministic_beta_U}$. We thus still need to condition on 1(a), and 3(b). Conditioning on 1(a) can be done using the independence of $\bm z$ and $\tilde{X}$. To condition on 3(b), we rewrite it in terms of $\tilde{X}$, i.e., for all $u\in \mathcal{U}$ we have
    \begin{align*}
        \bm\beta^*_u > 0 \Leftrightarrow \left(\pseudo \bm\tau\right)_u \Leftrightarrow \left((\pseudo)^{\frac{1}{2}} \tilde{X}\right)_u > 0 \Leftrightarrow \left((\pseudo)^{\frac{1}{2}} \frac{\tilde{X}}{\|\tilde{X}\|}\right)_u > 0.
    \end{align*}
    Thus it only depends on the direction of $\tilde{X}$. Since the non-trivial entries of $\tilde{X}$ follow a standard normal, the direction of $\tilde{X}$ is independent of its norm, i.e.,
    \begin{align}\label{eq:independence_of_direction_and_norm}
        \|\tilde{X}\|_2 \perp \!\!\! \perp 
    \frac{\tilde{X}}{\|\tilde{X}\|_2}.
    \end{align}
    In the end we get
    \begin{align*}
        &\left[\frac{\bm\beta^*\bm\tau}{(\bm\beta^*\Sigma\bm\beta^*)^{\frac{1}{2}}} \big| \text{Conditions } 1,2,3\right]
       \overset{\eqref{eq:deterministic_beta_U}}{\rightarrow} \overset{d}{=} \left[\frac{\bm\tau^\top \pseudo \bm\tau}{(\bm\tau\pseudo\bm\tau)^{\frac{1}{2}}} \big| \text{Conditions } 1(a),3(b)\right]\\
        \overset{d}{=} &\left[ \|\tilde{X}\|_2  \big|\begin{cases} z_u \leq 0 \quad \forall u\notin \mathcal{U},\\
        \left((\pseudo)^{\frac{1}{2}} \frac{\tilde{X}}{\|\tilde{X}\|}\right)_u > 0 \quad \forall u \in \mathcal{U}
        \end{cases}\right]
        \underset{\eqref{eq:independence_of_direction_and_norm}}{\overset{\eqref{eq:independence_X.z}}{\rightarrow}}\overset{d}{=} \left[ \|\tilde{X}\|_2\right]
        \overset{d}{=} \chi_l.
    \end{align*} 
    \end{proof}
    
    \begin{figure}[t]
    \centering
  
\centerline{\includegraphics[width=0.5\textwidth]{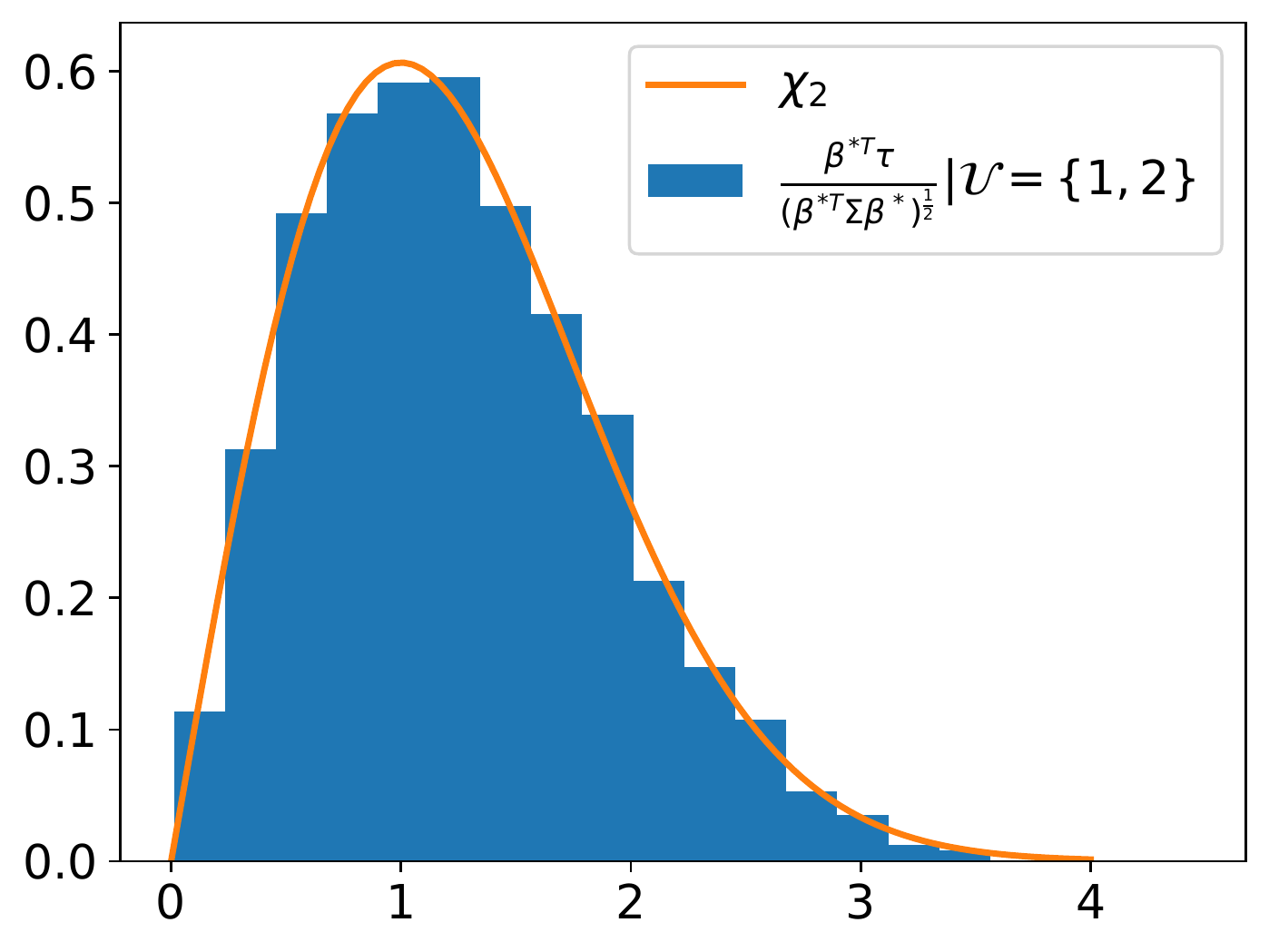}}
\caption{Numerical verification of Theorem \ref{thm:continuous}. For the histogram, we generate a random covariance matrix $\Sigma\in \mathbb{R}^{4\times 4}$ 
and sample $\bm\tau \sim \mathcal{N}(\bm 0, \Sigma)$. We solve \eqref{eq:objective} and only accept the samples for which the active set is $\mathcal{U}=\{1,2\}$. The orange line is the theoretical distribution according to Theorem \ref{thm:continuous}, which is given by a chi distribution with two degrees of freedom. For the specific example the acceptance rate is $P(\mathcal{U}=\{1,2\}) \approx 4\%$.
}
\label{fig:hist}
\end{figure}
    
    \subsection{Proof of Lemma \ref{lemma:conditions_optimum}}\label{app:proof_lemma_conditions}
    \begin{proof}[Proof of Lemma \ref{lemma:conditions_optimum}]
    Since the objective is a homogeneous function of order zero in $\bm\beta$, we can make the proof by considering the optimization without the constraint $\|\bm\beta\| =1$.
    
    The necessity of the conditions is trivial to show. We thus only show the sufficiency. The fourth condition ensures that $\bm\beta^*$ is in the feasible set. For the other conditions, assume there exists $\bm\xi \in \mathbb{R}^d$ such that $\xi_u\geq 0$ for all $u\in [d]$ and $\frac{\bm\xi^\top\bm \tau }{\left(\bm \xi^\top\Sigma \bm \xi\right)^\frac{1}{2}} > \frac{\bm\beta^{*\top}\bm \tau }{\left(\bm \beta^{*\top}\Sigma \bm \beta^*\right)^\frac{1}{2}}$. In the following we show that this implies that at least one of the conditions above is violated, and hence the conditions are sufficient. We separate two cases, $i)$ where ${\bm\beta^*}^\top\bm \tau  \geq 0$, and $ii)$ ${\bm\beta^*}^\top\bm \tau  < 0$.
    \begin{enumerate}
        \item[i)] Assume ${\bm\beta^*}^\top\bm \tau  \geq 0$. We have
            \begin{align*}
                &\bm\xi^\top \nabla_{\bm\beta}\left. \frac{\bm\beta^\top\bm \tau }{\left(\bm \beta^\top\Sigma \bm \beta\right)^\frac{1}{2}} \right|_{\bm\beta=\bm\beta^*}  \\
                &=   \sum_{u \in [d]} \xi_u \left. \frac{\partial}{\partial\beta_u} \frac{\bm\beta^\top\bm \tau }{\left(\bm \beta^\top\Sigma \bm \beta\right)^\frac{1}{2}} \right|_{\bm\beta=\bm\beta^*} \\
                &= \frac{{\bm\xi}^\top \bm \tau }{\left(\bm \beta^{*\top}\Sigma \bm \beta^*\right)^\frac{1}{2}} - \frac{{\bm\beta^*}^\top \bm \tau }{\left(\bm{\beta}^{*\top}\Sigma \bm\beta^*\right)^\frac{3}{2}}  \bm \xi^\top\Sigma {\bm\beta^*}  \\
                &= \frac{{\left(\bm{\xi^\top}\Sigma \bm\xi\right)^\frac{1}{2}}}{{\left(\bm{\beta}^{*\top}\Sigma \bm\beta^*\right)^\frac{1}{2}}} 
                 \left(\frac{\bm\xi^\top\bm \tau }{\left(\bm \xi^\top\Sigma \bm \xi\right)^\frac{1}{2}} - \frac{{\bm\beta^*}^\top \bm \tau }{\left(\bm{\beta}^{*\top}\Sigma \bm\beta^*\right)^\frac{1}{2}}  \frac{\bm \xi^\top\Sigma {\bm\beta^*}}{\left(\bm{\beta}^{*\top}\Sigma \bm\beta^*\right)^\frac{1}{2} \left(\bm \xi^\top\Sigma \bm \xi\right)^\frac{1}{2}}   \right)\\
                &>  \frac{{\left(\bm{\xi^\top}\Sigma \bm\xi\right)^\frac{1}{2}}}{{\left(\bm{\beta}^{*\top}\Sigma \bm\beta^*\right)^\frac{1}{2}}} 
                 \left(\frac{{\bm\beta^*}^\top \bm \tau }{\left(\bm{\beta}^{*\top}\Sigma \bm\beta^*\right)^\frac{1}{2}}  - \frac{{\bm\beta^*}^\top \bm \tau }{\left(\bm{\beta}^{*\top}\Sigma \bm\beta^*\right)^\frac{1}{2}}  \frac{\bm \xi^\top\Sigma {\bm\beta^*}}{\left(\bm{\beta}^{*\top}\Sigma \bm\beta^*\right)^\frac{1}{2} \left(\bm \xi^\top\Sigma \bm \xi\right)^\frac{1}{2}}   \right)\\
                &=\frac{{\left(\bm{\xi^\top}\Sigma \bm\xi\right)^\frac{1}{2}}}{{\left(\bm{\beta}^{*\top}\Sigma \bm\beta^*\right)^\frac{1}{2}}} \frac{{\bm\beta^*}^\top \bm \tau }{\left(\bm{\beta}^{*\top}\Sigma \bm\beta^*\right)^\frac{1}{2}}
                 \left(1  - \frac{\bm \xi^\top\Sigma {\bm\beta^*}}{\left(\bm{\beta}^{*\top}\Sigma \bm\beta^*\right)^\frac{1}{2} \left(\bm \xi^\top\Sigma \bm \xi\right)^\frac{1}{2}}   \right)\\
                &\geq\ 0,
            \end{align*}
            where we used the assumption $\frac{\bm\xi^\top\bm \tau }{\left(\bm \xi^\top\Sigma \bm \xi\right)^\frac{1}{2}} > \frac{\bm\beta^{*\top}\bm \tau }{\left(\bm \beta^{*\top}\Sigma \bm \beta^*\right)^\frac{1}{2}}$ for the first inequality and 
           ${\bm\beta^*}^\top\bm \tau  \geq 0$ and the Cauchy-Schwarz inequality to arrive at the last line. 
            Since, by assumption, $\xi_u \geq 0$ for all $u$, this implies $ \left. \frac{\partial}{\partial\beta_u} \frac{\bm\beta^\top\bm \tau }{\left(\bm \beta^\top\Sigma \bm \beta\right)^\frac{1}{2}} \right|_{\bm\beta=\bm\beta^*} > 0$ for some $u$ and thus is a contradiction to Condition 1.
            
        \item[ii)]  Assume ${\bm\beta^*}^\top\bm \tau < 0$. We define $u^* = \argmax_{u \in [d]} \frac{\tau_{u} }{\left(e_{u}^\top\Sigma e_{u}\right)^\frac{1}{2}}$. By the third condition and the assumption ${\bm\beta^*}^\top\bm \tau < 0$, we have $ 0 > \frac{\bm\beta^{*\top}\bm \tau }{\left(\bm \beta^{*\top}\Sigma \bm \beta^*\right)^\frac{1}{2}} \geq  \frac{\tau_{u^*} }{\left(e_{u^*}^\top\Sigma e_{u^*}\right)^\frac{1}{2}} $. This implies $\tau_{u^*} < 0$. We then get
        \begin{align*}
            \frac{\bm\xi^\top\bm \tau }{\left(\bm \xi^\top\Sigma \bm \xi\right)^\frac{1}{2}} &=  \sum_{u\in[d]}\xi_u\frac{\tau_u }{\left(\bm \xi^\top\Sigma \bm \xi\right)^\frac{1}{2}} = \sum_{u\in[d]}\xi_u\frac{\tau_u \left(e_{u}^\top\Sigma e_{u}\right)^\frac{1}{2}}{\left(\bm \xi^\top\Sigma \bm \xi\right)^\frac{1}{2}\left(e_{u}^\top\Sigma e_{u}\right)^\frac{1}{2}}\\
            &\leq   \sum_{u\in[d]}\xi_u\frac{\tau_{u^*} \left(e_{u}^\top\Sigma e_{u}\right)^\frac{1}{2} }{\left(e_{u^*}^\top\Sigma e_{u^*}\right)^\frac{1}{2}\left(\bm \xi^\top\Sigma \bm \xi\right)^\frac{1}{2}}\\ 
            &= \frac{\tau_{u^*} }{\left(e_{u^*}^\top\Sigma e_{u^*}\right)^\frac{1}{2}}  \frac{\sum_{u\in[d]}\xi_u{\left(e_{u}^\top\Sigma e_{u}\right)^\frac{1}{2}}}{\left(\bm \xi^\top\Sigma \bm \xi\right)^\frac{1}{2}}\\
            & \leq \frac{\tau_{u^*} }{\left(e_{u^*}^\top\Sigma e_{u^*}\right)^\frac{1}{2}} \leq  \frac{\bm\beta^{*\top}\bm \tau }{\left(\bm \beta^{*\top}\Sigma \bm \beta^*\right)^\frac{1}{2}} ,
        \end{align*}
        where to arrive at the last line we used $\tau_{u^*}  < 0$ and  the triangle inequality $\sum_{u\in[d]}\xi_u{\left(e_{u}^\top\Sigma e_{u}\right)^\frac{1}{2}} = \sum_{u\in[d]}\xi_u \|\Sigma^\frac{1}{2}  e_u\| \geq  \|\sum_{u\in[d]} \xi_u \Sigma^\frac{1}{2}  e_u\| =  \| \Sigma^\frac{1}{2} \bm\xi\| = {\left(\bm \xi^\top\Sigma \bm \xi\right)^\frac{1}{2}}$. Thus this violates the assumption $\frac{\bm\xi^\top\bm \tau }{\left(\bm \xi^\top\Sigma \bm \xi\right)^\frac{1}{2}} > \frac{\bm\beta^{*\top}\bm \tau }{\left(\bm \beta^{*\top}\Sigma \bm \beta^*\right)^\frac{1}{2}}$.
        
        Note that the above inequalities also hold for $\frac{\bm\beta^{*\top}\bm \tau }{\left(\bm \beta^{*\top}\Sigma \bm \beta^*\right)^\frac{1}{2}}$. Thus we get that $\frac{\bm\beta^{*\top}\bm \tau }{\left(\bm \beta^{*\top}\Sigma \bm \beta^*\right)^\frac{1}{2}} =  \frac{\tau_{u^*} }{\left(e_{u^*}^\top\Sigma e_{u^*}\right)^\frac{1}{2}}$. This implies that $l=|\mathcal{U}| = 1$. Thus the following statements hold true:
        \begin{align}
            i)& \qquad \bm\beta^{*\top}\bm\tau < 0 \quad \Rightarrow \quad l=1,\\
            ii)& \qquad l\l\geq 2 \,\qquad\quad\Rightarrow \quad  \bm\beta^{*\top}\bm\tau \geq 0. \label{eq:positive_if_l>2}
        \end{align}
    \end{enumerate}
\end{proof}

\subsection{Gradient of objective}\label{app:gradient}
We overload the notation and define $\bm z := \bm \tau - \Sigma\bm\beta \frac{\bm\beta^\top \bm\tau}{\bm{\beta}^\top\Sigma \bm{\beta}}$ similar as in \eqref{eq:definition_z_cont} but for any $\bm\beta$. Then
\begin{align*}
    \nabla_{\bm{\beta}} f(\bm\beta)
    &= \nabla_{\bm{\beta}}\left(\frac{\bm\beta^\top\bm \tau }{\left(\bm \beta^\top\Sigma \bm \beta\right)^\frac{1}{2}}\right) \\
    &= \frac{(\bm \beta^\top\Sigma \bm \beta)^\frac{1}{2}\nabla_{\bm{\beta}}(\bm\beta^\top\bm \tau) - \bm\beta^\top\bm \tau \nabla_{\bm{\beta}}((\bm \beta^\top\Sigma \bm \beta)^\frac{1}{2})}{\bm\beta^\top\Sigma\bm\beta} \\
    &= \frac{(\bm \beta^\top\Sigma \bm \beta)^\frac{1}{2}\bm \tau - \frac{1}{2}\bm\beta^\top\bm \tau ((\bm \beta^\top\Sigma \bm \beta)^{-\frac{1}{2}})\cdot 2\bm\beta^\top\Sigma}{\bm\beta^\top\Sigma\bm\beta} \\
    &= \frac{1}{(\bm\beta^\top\Sigma\bm\beta)^{\frac{1}{2}}}\left(\bm\tau - \Sigma\bm\beta\left(\frac{\bm\beta^\top\bm\tau}{(\bm\beta^\top\Sigma\bm\beta)}\right)\right)\numberthis{}\label{eq:44} \\
    &= \frac{1}{(\bm\beta^\top\Sigma\bm\beta)^{\frac{1}{2}}}\bm{z}.
\end{align*}

\subsection{Proof of Equation \eqref{eq:independence_X.z}}\label{app:correlation_X,z}
In the proof of Theorem \ref{thm:continuous} we used that $\tilde{X}$ and $\bm z$ are independent. Which we prove here.
Since $\tilde{X}$ and $\bm z$ are jointly normal, we only need to show that they are uncorrelated. To do so recall that we are only interested in the distribution under the null and hence $\bm 0 = \expect{}{\bm \tau} = \expect{}{\tilde{X}} = \expect{}{\bm z}$. Since $\tilde{X}_u = 0$ for all $u \notin \mathcal{U}$ and $\bm z _u'=0$ for all $u'\in \mathcal{U}$, it suffices to show that $\tilde{X}_j$ is uncorrelated with $z_i$ for all $i \notin \mathcal{U}$, $j\in \mathcal{U}$.
\begin{align*}
    &\text{Cov}\left[ z_i, \tilde{X}_j \right] = \expect{}{z_i \tilde{X}_j} = \expect{}{\left( \tau_i - (\Sigma \pseudo\bm\tau)_i \right) ((\pseudo)^\frac{1}{2} \bm\tau)_j} \\
    &= \sum_{u\in \mathcal{U}} ((\pseudo)^\frac{1}{2})_{ju} \expect{}{\tau_i, \tau_u} - 
    \sum_{\substack{s,t,u \in \mathcal{U}}} ((\pseudo)^\frac{1}{2})_{ju}\Sigma_{is} \pseudo_{st} \expect{}{\tau_t \tau_u}\\
    &= \sum_{u\in \mathcal{U}} ((\pseudo)^\frac{1}{2})_{ju} \Sigma_{iu} - 
    \sum_{\substack{s,t,u \in \mathcal{U}}} ((\pseudo)^\frac{1}{2})_{ju}\Sigma_{is} \pseudo_{st} \Sigma_{tu}\\
    &= \left((\pseudo)^\frac{1}{2} \Sigma \right)_{ji} -  \left(\Sigma \pseudo \Sigma (\pseudo)^\frac{1}{2} \right)_{ij} \\
    &= \left((\pseudo)^\frac{1}{2} \Sigma \right)_{ji} -  \left(\Sigma  (\pseudo)^\frac{1}{2} \right)_{ij} = 0.
\end{align*}
Thus $\tilde{X}$ and $\bm z$ are uncorrelated and independent.

\section{Solution of the continuous optimization problem}\label{sec:optimization}
The presented solution is similarly described in \citet[Sec.~4]{Gretton2012optimal}. There an $L1$ norm constraint was used, which, however does not change anything. For completeness we include it here. We define
\begin{align*}
    f(\bm\beta) := \frac{\bm\beta^\top\bm\tau}{(\bm{\beta}^\top\Sigma \bm{\beta})^\frac{1}{2}},
\end{align*}
and we want to find 
\begin{align*}
    \bm\beta^* = \argmax_{\bm\beta \geq \bm 0, \|\bm\beta\|=1} \frac{\bm\beta^\top\bm\tau}{(\bm{\beta}^\top\Sigma \bm{\beta})^\frac{1}{2}}.
\end{align*}
Since $f$ is a homogeneous function of order 0 in $\bm\beta$ we have $f(c \bm\beta) = f(\bm\beta)$ for any $c > 0$. We can thus solve the relaxed problem (we implicitly exclude $\bm\beta = \bm 0$)
\begin{align*}
    \bm\beta' = \argmax_{\bm\beta \geq \bm 0} f(\bm\beta).
\end{align*}
The solution of the original problem is then simply given as a rescaled version of the relaxed problem $\bm\beta^* = \frac{\bm\beta'}{\|\bm\beta'\|}$. We shall solve the relaxed problem for two different cases.
\begin{enumerate}
    \item[i)] $\exists u \in [d]: \tau_u \geq 0$.
    
    In this case, we know that $\max_{\bm\beta \geq \bm 0} f(\bm\beta) \geq 0$ and hence $\bm\beta' = \argmax_{\bm\beta \geq \bm 0} f(\bm\beta) \Leftrightarrow \bm\beta' = \argmax_{\substack{\bm\beta \geq \bm 0\\
    f(\bm\beta)\geq 0}} f^2(\bm\beta)$. The set $S:=\{\bm\beta \in \mathbb{R}^d | \bm\beta \geq \bm 0,
    f(\bm\beta)\geq 0\}$ is convex and the functions $g_1(\bm\beta) := (\bm\beta^\top\bm\tau)^2$ and $g_2(\bm\beta) := \bm{\beta}^\top\Sigma \bm{\beta}$ are convex (recall that $\Sigma$ is a positive matrix). Thus our problem becomes
    \begin{align*}
        \bm\beta' = \argmax_{\bm\beta \in S} \frac{g_1(\bm\beta)}{g_2(\bm\beta)},
    \end{align*}
    which is a concave fractional program. 
    In our implementation we solve it by fixing $\bm\beta^\top\bm\tau = a$ for some $a>0$ and then minimizing the denominator. Thus we are solving the quadratic optimization problem
    \begin{align*}
        \text{minimize} &\quad\bm\beta^\top\Sigma \bm\beta\\
        \text{subject to:} &\quad{\bm\beta \geq \bm 0}\\
        &\quad\bm\beta^\top\bm\tau = a.
    \end{align*}
    We solve this problem with the CVXOPT python package \citep{vandenberghe2010cvxopt}.
    
    \item[ii)] $\tau_u < 0 \, \forall u \in [d]$.
    
    In this case we have ${\bm\beta^*}^\top\bm \tau < 0$. By \eqref{eq:positive_if_l>2} we have $l=1$.
    Thus we simply $\bm\beta^* = e_{u^*}$, where $u^* = \argmax_{u \in [d]} \frac{\tau_u}{\Sigma_{u,u}}$. 
    
\end{enumerate}
Note that in the case $\bm\tau = \bm 0$, $\bm\beta^* $ is not well defined and we could randomly select any $\bm\beta^*$. However, the probability of this happening is 0.

\section{Other proofs}
\subsection{Proof of Corollary \ref{cor:selection_discrete}}\label{app:proof_cor_discrete}
As we pointed out in the main paper, when selecting a test from a countable number of test that can be written as projections of the base tests $\bm\tau$ we can use the results of \citet{Lee2016}. For completeness we explicitly include the relevant theorem.
\begin{thm}[Polyhedral Lemma \citep{Lee2016}, Theorem 5.2]\label{thm:poly}
Let $\bm{\tau} \sim \mathcal{N}(\bm{\mu}, \Sigma)$, $\bm\eta, \bm\mu \in \mathbb{R}^d$, $\Sigma \in \mathbb{R}^{d\times d}$ positive definite, and $A \in \mathbb{R}^{s\times d}$, $\bm b \in \mathbb{R}^s$ for some $s\in \mathbb{N}$. Define $\bm{c} := \Sigma \bm\eta \left(\bm{\eta}^\top \Sigma \bm{\eta}\right)^{-1}$ and $ \bm{z} := \left(I_d - \bm{c}\bm{\eta}^\top\right) \bm{\tau}$. Then we have
\begin{align*}
        \left[\bm\eta^\top \bm\tau | A \bm\tau \leq \bm{b}, \bm z = \hat{\bm z}\right] 
    \overset{d}{=}  \text{TN}\left(\bm\eta^\top \bm\mu, \bm\eta^\top \Sigma \bm\eta, \mathcal{V}^- (\hat{\bm{z}}), \mathcal{V}^+(\hat{\bm{z}}) \right),
\end{align*}
where $\text{TN}(\mu, \sigma^2, a,b)$ denotes a Gaussian distribution with mean $\mu$ and variance $\sigma^2$ that is truncated at $a$ and $b$. Here 
\begin{align*}
    \mathcal{V}^- (\bm{z}) := \max_{j: (A\bm c)_j < 0} \frac{\bm b_j - (A\bm  z)_j}{(A\bm  c)_j}, \quad
     \mathcal{V}^+ (\bm{z}):= \min_{j: (A\bm c)_j > 0} \frac{\bm b_j - (A\bm  z)_j}{(A\bm  c)_j}.
\end{align*}
\end{thm}
Note that $\bm c$ is simply a fixed vector. $\bm z$ is a random variable that
can be shown to be independent of $\bm\eta^\top \bm\tau$.
The result enables us to draw a
realization $\hat{\bm\tau}$ of the random variable (RV) $\bm\tau$ and select
$\bm\eta$ if $A\hat{\bm\tau} \leq \bm b$. Since the truncation
points of the Gaussian only depend on $\hat{\bm z}$, and ${\bm z}$ is
independent of $\bm\eta^\top {\bm\tau}$, we can compute a reliable $p$-value
of $\bm\eta^\top \hat{\bm\tau}$ by using Theorem \eqref{thm:poly}.

\begin{proof}[Proof of Corollary \ref{cor:selection_discrete}]
We need the distribution of $\frac{\tau_{u^*}}{\sigma_{u^*}}$ after conditioning on the selection of $u^*$. 
To obtain this distribution we first need to characterize the event that leads to the selection of $u^*$. The selection event simply is $u^* = \argmax_{u \in [d]} \frac{\tau_u}{\sigma_u} \Leftrightarrow  \frac{\tau_{u^*}}{\sigma_{u^*}} \geq \frac{\tau_u}{\sigma_u} \text{ for all } u\in [d]$.
Therefore, define the matrix $A:= \text{diag}(\frac{1}{\sigma_1}, \dots, \frac{1}{\sigma_d}) - \frac{1}{\sigma_{u^*}} A(u^*)$, where $\text{diag}(\cdot)$ defines a $d\times d$ matrix with the arguments on its diagonal and zeros everywhere else and $A(\cdot)$ is a $d\times d$ matrix with ones in the column given by its argument and zeros everywhere else.
It follows that
$
    (A\bm\tau)_j = \frac{\tau_j}{\sigma_j} - \frac{\tau_{u^*}}{\sigma_{u^*}},
$
and $u^* = \argmax_{u \in [d]} \frac{\tau_u}{\sigma_u}$ is equivalent to $A{\bm\tau} \leq \bm 0 =: \bm b$. Apart from this we define $\bm\eta := \frac{{e}_{u^*}}{\sigma_{u^*}}$, so that $\bm\eta^\top \bm\tau = \frac{\tau_{u^*}}{\sigma_{u^*}}$. Then we can define $\bm{c} := \Sigma \bm\eta \left(\bm{\eta}^\top \Sigma \bm{\eta}\right)^{-1}$ and $ \bm{z} := \left(I_d - \bm{c}\bm{\eta}^\top\right) \bm{\tau}$ as in Theorem \ref{thm:poly}, and denote by $\hat{\bm z}$ the value of the random variable $\bm z$ that we observed (note that this coincides with the definition we used for $\bm z$ in the Corollary). 
By our definitions we have $(A\bm c)_j = \frac{\Sigma_{ju^*}/\sigma_j - \sigma_{u^*}}{\sigma_{u^*}} = \frac{1}{\sigma_{u^*}\sigma_j} \left(\Sigma_{u^* j} - \sigma_u^* \sigma_j\right)$. 
Since $\Sigma$ is positive definite, $(A\bm c)_j <
0$ if $j\neq u^*$ and $(A\bm c)_{u^*}= 0$. Thus according to Theorem
\ref{thm:poly}, $\mathcal{V}^+$ is an optimization over an empty set and we can set it
to $\infty$.
    Further $(A\bm z)_j = \frac{1}{\sigma_{u^*}\sigma_j} \left(\tau_j \sigma_{u^*} - \frac{\Sigma_{u^*j}}{\sigma_{u^*}}\tau_{u^*}\right)$. Combining the previous two expressions we obtain
    $\frac{- (A\bm z)_j}{(A\bm c)_j} = \frac{\tau_j \sigma_{u^*} -\frac{ \Sigma_{u^*j}}{\sigma_{u^*}}\tau_{u^*}}{\sigma_u^* \sigma_j - \Sigma_{u^* j}} = \frac{\sigma_{u^*} z_j}{\sigma_u^* \sigma_j - \Sigma_{u^* j}} $. We can then directly apply Theorem \ref{thm:poly} and the result follows.
\end{proof}

\subsection{Proof of Equation \eqref{eq:beta_inf}}\label{sec:proof_beta_inf}
In the main paper we omitted the proof of the closed form solution of $\bm\beta^\infty$. 
We thus need to show 
\begin{align*}
    \argmax_{\|\bm\beta\| = 1} \frac{\bm\beta^\top \bm\mu}{(\bm\beta^\top \Sigma \bm\beta)^\frac{1}{2}} = \frac{\Sigma^{-1}\bm\mu}{\|\Sigma^{-1}\bm\mu\|}.
\end{align*}
\begin{proof}
We are only interested in $\bm\beta^\infty$ if the alternative hypothesis is
true and thus at least one entry of $\bm\mu$ is positive. We further assume
that the covariance $\Sigma$ has full rank. Hence there exists a $b >0$ such
that $\bm\beta^\top \Sigma \bm\beta > b$  for all $\bm\beta$ with $\|\bm\beta\| =1$, i.e., the
denominator $(\bm\beta^\top \Sigma \bm\beta)^\frac{1}{2}$  
is strictly positive and has a lower bound. Since $\bm\mu \neq \bm 0$, this implies that $\max_{\|\bm\beta\| = 1} \frac{\bm\beta^\top \bm\mu}{(\bm\beta^\top \Sigma \bm\beta)^\frac{1}{2}} > 0$.  Also the nominator has an upper
bound which is given by $\bm\beta^\top \bm\mu \leq \bm\mu^\top \bm\mu /
\|\bm\mu\|$ if $\|\bm\beta\| =1$. Hence the whole maximization is upper
bounded. Since the unit sphere in $\mathbb{R}^d$ is a compact set, we can
conclude that the maximum of the objective is attained. Thus it suffices to
show that for all $\bm\beta\neq \bm\beta^\infty$ the objective is not
maximized. In the following, we use that the objective of the maximization is a
homogeneous function of order 0 in $\bm\beta$ and hence we can relax the
constraint $\|\bm\beta\|=1$ to $\bm\beta \neq \bm 0$ (note that this not
affect the existence of the maximum).
As we showed in Appendix \ref{app:gradient}, the gradient of the objective function is given by 
\begin{align*}
    \nabla_{\bm\beta} \frac{\bm\beta^\top \bm\mu}{(\bm\beta^\top \Sigma \bm\beta)^\frac{1}{2}} = \frac{1}{(\bm\beta^\top\Sigma\bm\beta)^{\frac{1}{2}}}\left(\bm\mu - \Sigma\bm\beta\left(\frac{\bm\beta^\top\bm\mu}{(\bm\beta^\top\Sigma\bm\beta)}\right)\right).
\end{align*}
Setting the gradient to zero we obtain
\begin{align*}
    \nabla_{\bm\beta} \frac{\bm\beta^\top \bm\mu}{(\bm\beta^\top \Sigma \bm\beta)^\frac{1}{2}} = \bm 0 \Leftrightarrow \bm\beta = c  \cdot {\Sigma^{-1}\bm\mu} \text{ for some } c \in \mathbb{R}.
\end{align*}
If $c < 0$ the objective attains a negative value, since $\Sigma^{-1}$ is a strictly positive matrix, and thus does not correspond to the global maximum, which we already know to be positive. Thus, the maximum has to be attained for some $c > 0$. Using the constraint $\|\bm\beta\| =1$ it follows that the global optimum is attained at $\bm\beta^\infty$.
\end{proof}

\section{Experimental details and further experiments}\label{app:Experimental_Details}
We first give some details on the experiments we showed in the main paper. 
For all the experiments we start with a set of $d$ base kernels $\mathcal{K} = [k_1,\dots, k_d]$ that are chosen independently of the observed data samples $X = \{x_1,\dots,x_{2n}\} \sim P^{2n}$ and $Y=\{y_1,\dots,y_{2n}\}\sim Q^{2n}$. 
First, we define $z_i := (x_i, x_{n+i}, y_i, y_{n+i})$ and compile $X$ and $Y$ into $Z = \{z_1, \dots, z_n\}$. For each kernel we define $h_i (z) := h_i(x,x',y,y') := k_i(x, x') + k_i(y, y') - k_i(x, y') - k_i(y, x')$. For all the methods we estimate the covariance matrix on the whole dataset as
\begin{align*}
    \hat{\Sigma}_{ij} = \frac{1}{n}\sum_{k=1}^n h_i(z_k)h_j(z_k) -  \frac{1}{n}\sum_{k=1}^n h_i(z_k)\frac{1}{n}\sum_{k'=1}^n h_j(z_{k'}).
\end{align*}
We then further assume that $\Sigma=\hat{\Sigma}$ which is justified since the CLT also works with a consistent estimate of the covariance.
For all the methods that do not split the data (\textsc{\name}, \textsc{Wald}, and \textsc{Naive}) we estimate the entries of $\hat{\bm\tau}$ as 
\begin{align*}
    \hat\tau_i = \sqrt{n} \,\widehat{\text{MMD}}^2_\text{lin} (P,Q) = \sqrt{n}\, \frac{1}{n} \sum_{k=1}^n h_i(z_k),
\end{align*}
i.e., we directly absorb the $\sqrt{n}$ dependence of the asymptotic distribution into $\bm\tau$.
For data splitting we estimate $\hat{\bm\tau}_\text{tr}$ on a split of the data and $\hat{\bm\tau}_\text{te}$ on the other split. For example \textsc{split0.3} means that $30\%$ of the data are used to estimate $\hat{\bm\tau}_\text{tr}$ and $70\%$ used to estimate $\hat{\bm\tau}_\text{te}$. We assume that the number of samples in the respective subsets are even and otherwise neglect some samples.

\vspace{-5pt}
\paragraph{Methods} We compare four different methods:
\begin{enumerate}[i)]
    \item \textsc{\name}: The test we recommend to use, as described in  Algorithm \ref{alg:continuous}.
    \item \textsc{Wald}: The Wald test, which does not take into account the prior information $\bm\mu \geq \bm 0$.
    \item \textsc{split}: Data splitting similar to the approach in \citet{Gretton2012optimal}. \textsc{split0.3} denotes that 30\% of the data are used for learning $\bm\beta^*$ and 70\% are used for testing. Here we first, learn $\bm\beta^*$ on the training sample, i.e., ${\bm\beta}^* =  \argmax_{\|\Sigma\bm\beta\|=1,\Sigma\bm\beta \geq 0} \frac{\bm{\beta}^\top \bm\tau_{tr}}{(\bm{\beta}^\top \Sigma \bm{\beta})^\frac{1}{2}}$. We then use the test statistic $\frac{\hat{\bm{\beta}}^\top \bm\tau_{te}}{(\hat{\bm{\beta}}^\top \Sigma \hat{\bm{\beta}})^\frac{1}{2}}$, which follows a standard normal under the null. This differs from the approach in \citet{Gretton2012optimal}, since we optimize with the constraints $\Sigma \bm\beta \geq \bm 0$, whereas \citet{Gretton2012optimal} suggested a simple positivity constraint $\bm\beta \geq \bm 0$. We discuss this in Section \ref{app:constraints}.
    \item \textsc{naive}: Two stage procedure where all the data is used for learning and testing without correcting for the dependency, i.e., without splitting the data. Thus the test statistic is the same as for \textsc{\name}, but we work with the wrong null distribution, i.e., the one that is only valid for data splitting. This approach is not a well-calibrated test, see Fig. \ref{fig:typeI} and hence is useless.
\end{enumerate}

\vspace{-5pt}
\paragraph{Datasets}
The \texttt{DIFF VAR} dataset is a simple one-dimensional toy dataset, where $P = \mathcal{N}(0,1)$ and $Q=\mathcal{N}(0,1.5)$.

The \texttt{Blobs} dataset was constructed using a mixture of 2D Gaussians on a $3\times 3$ grid. The centers of the Gaussians are set to $\mu_1, \dots, \mu_9 = (0,0),(0,1),(0,2), (1,0),(1,1),(1,2), (2,0),(2,1),(2,2)$ and the covariances are $\Sigma_P = \text{diag}(0.1, 0.3)$ and $\Sigma_Q = \text{diag}(0.3, 0.1)$.
Samples from $P$ and $Q$ are shown in Figure \ref{fig:blobs_samples}. The \texttt{Blobs} dataset is constructed such that the main variance in the data does not reflect the difference between $P$ and $Q$, which happens on a smaller length scale. This is inspired by \citet{Gretton2012optimal}, where similar data has been considered to showcase that such problems benefit from careful kernel choice. We can reproduce this behavior with our results, which show that for this dataset the performance is bad if one only considers the median heuristic Gaussian kernel together with a linear kernel. 
\begin{figure}[t]
\centering
\begin{subfigure}{}
    \centering
    \includegraphics[width=.45\linewidth]{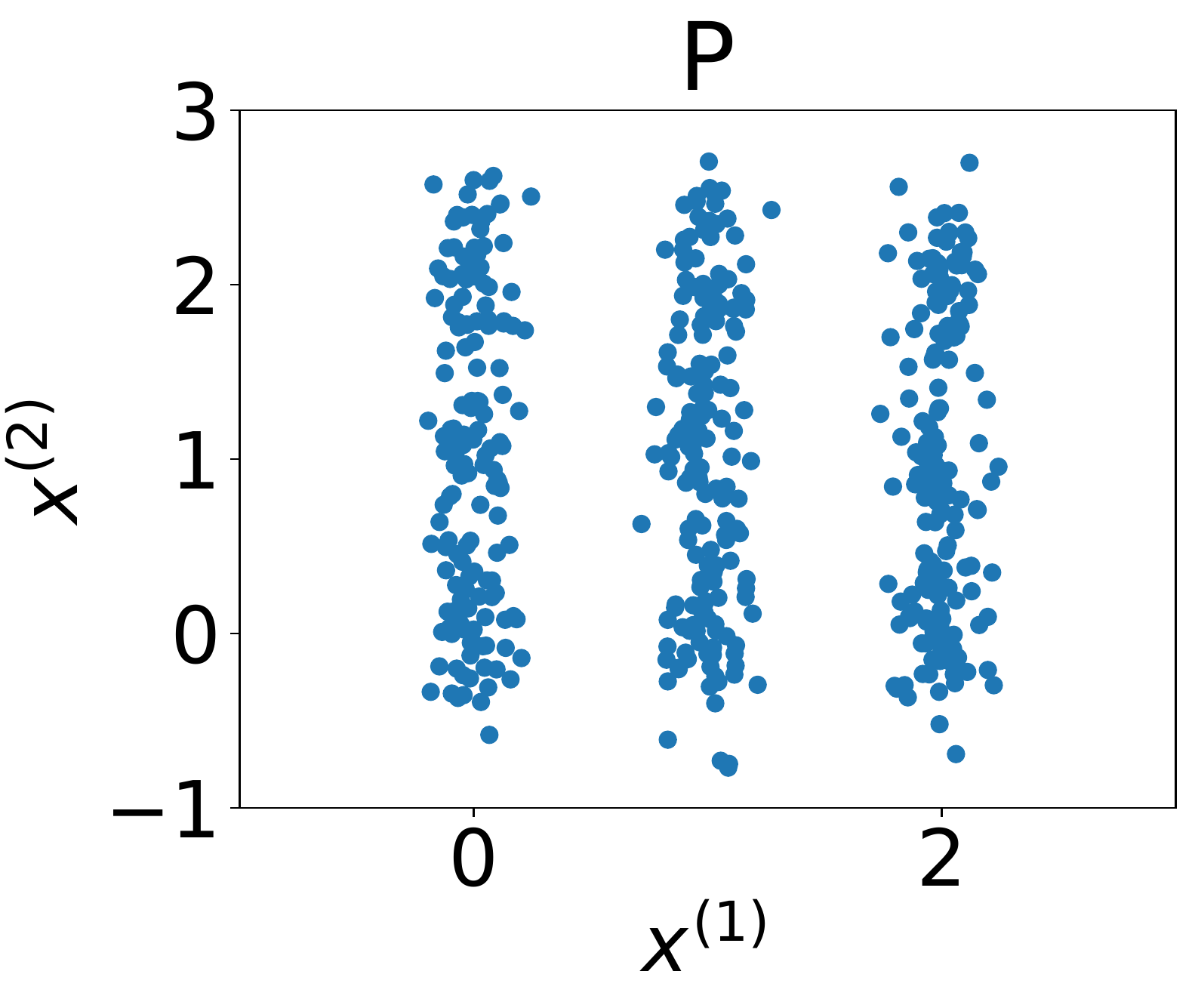}
\end{subfigure}
\begin{subfigure}{}
    \centering
    \includegraphics[width=.45\linewidth]{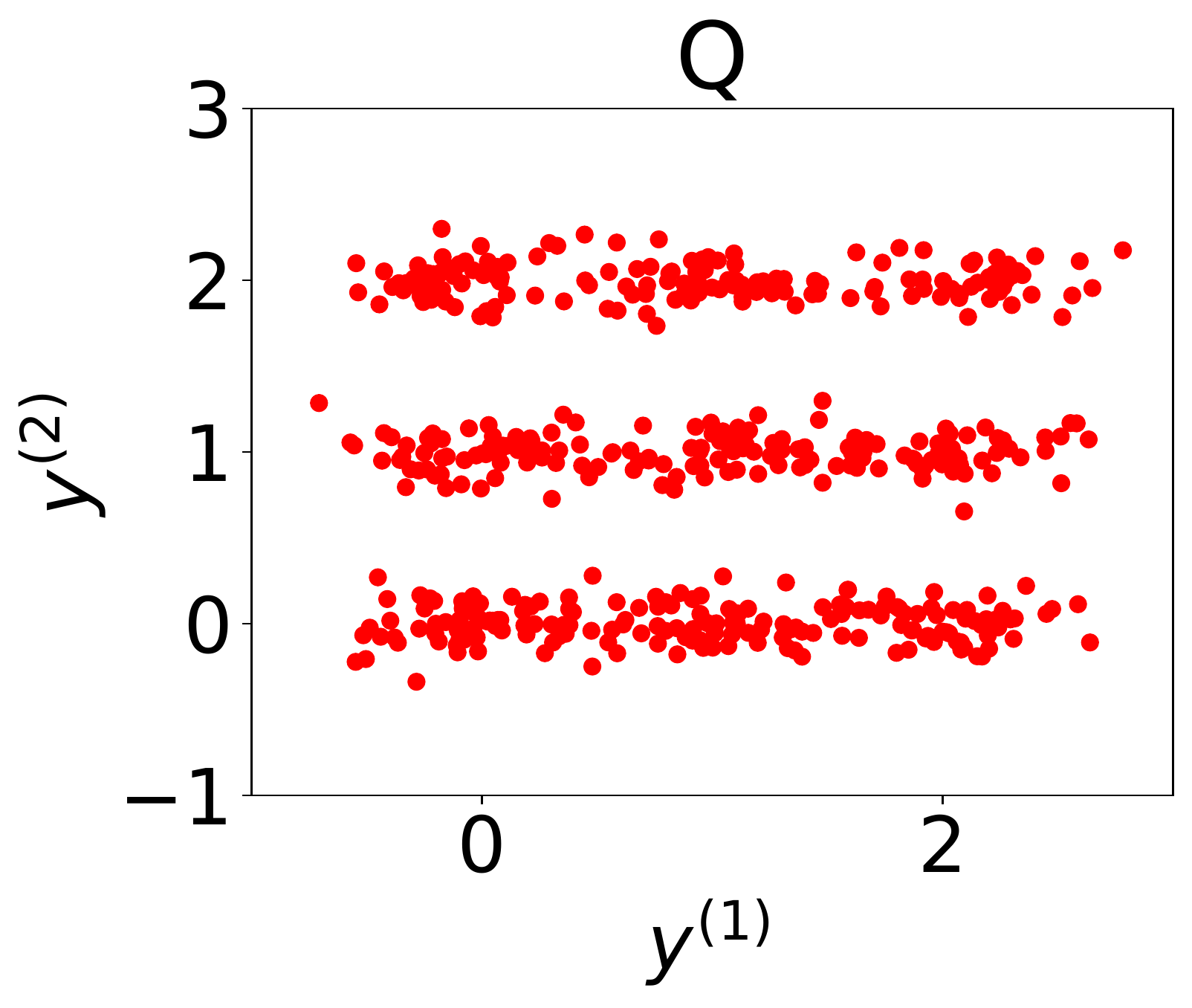}
\end{subfigure}
\caption{Samples from \textsc{Blobs} dataset.}
\label{fig:blobs_samples}
\end{figure}

The \texttt{MNIST} dataset was constructed by first downsampling all the images to $7\times 7$ pixels (originally $28\times 28$), by simply averaging over fields of $4\times 4$ pixels. We define $P$ to contain all the digits, while $Q$ only contains uneven digits. For our experiments we draw with replacement from the images in the database. Some samples from both distributions are shown in Figure \ref{fig:mnist_samples}.

\begin{figure}[t]
\centering

\begin{subfigure}{}
    \centering
    \includegraphics[width=.45\linewidth]{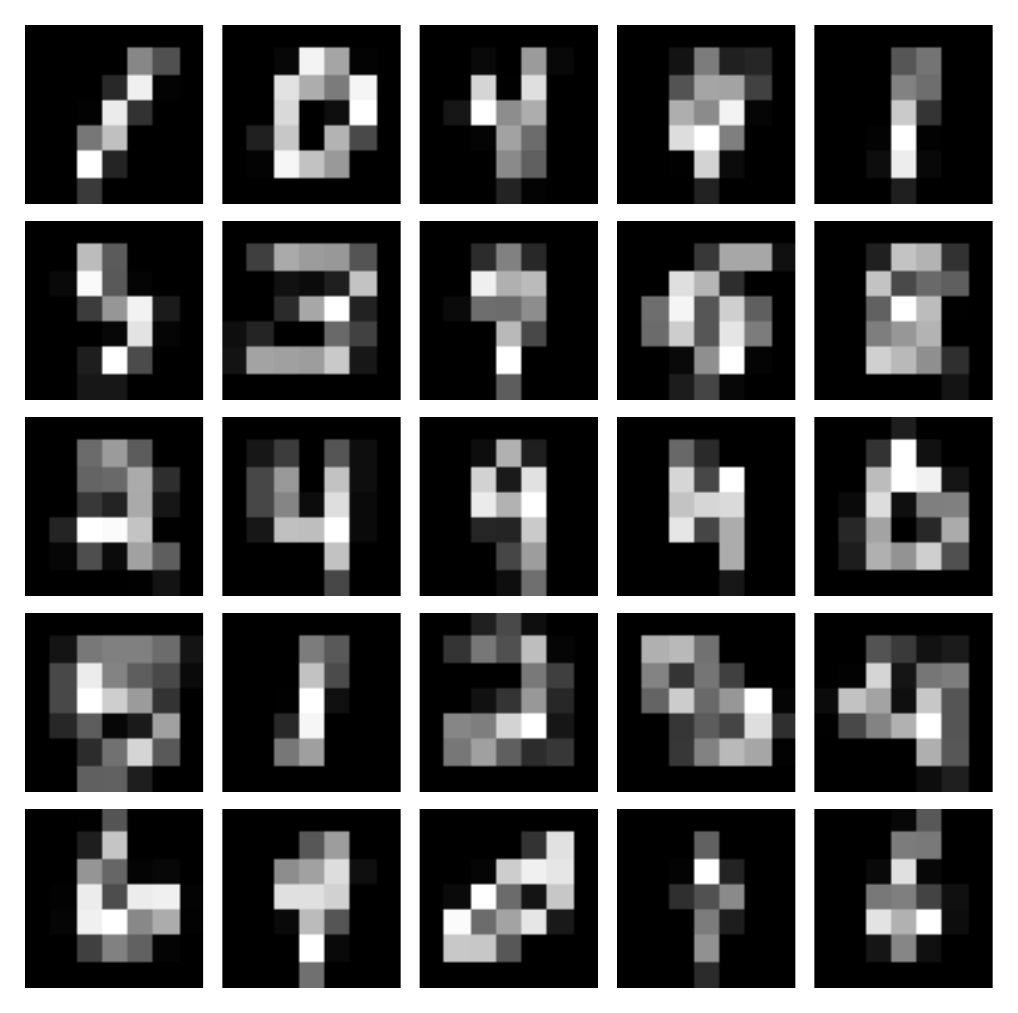}
\end{subfigure}
\begin{subfigure}{}
    \centering
    \includegraphics[width=.45\linewidth]{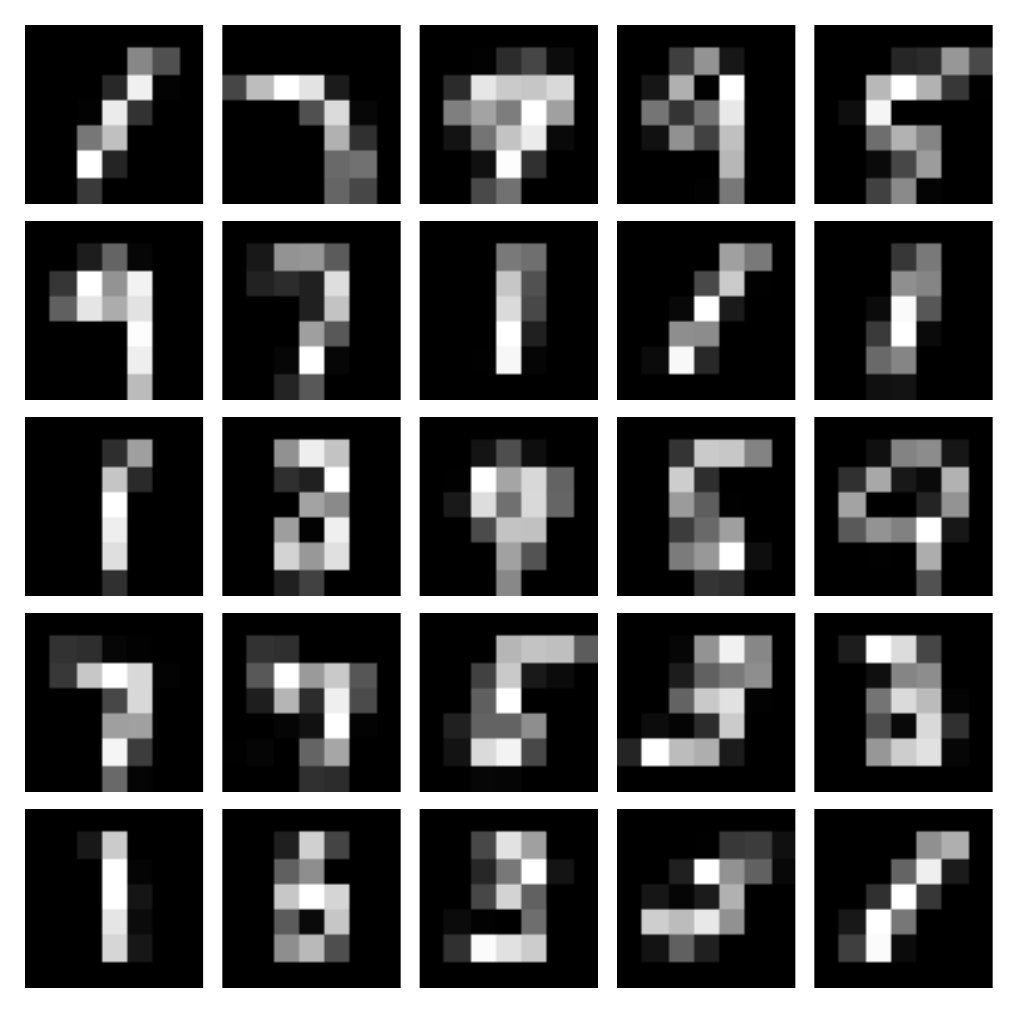}
\end{subfigure}
\caption{Samples from downsampled \textsc{MNIST} dataset. $P$ (left) contains all digits, while $Q$ (right) only contains uneven digits.}
\label{fig:mnist_samples}
\end{figure}

\paragraph{Experiments for Figure \ref{fig:laplaceII}} For Figure \ref{fig:laplaceII} we constructed a $1$-D data set such that both $P$ and $Q$ are symmetric (thus all uneven moments vanish) and have the same variance, see Figure \ref{fig:plot_uniform}.
\begin{figure}
    \centering
    \includegraphics[width=0.5\linewidth]{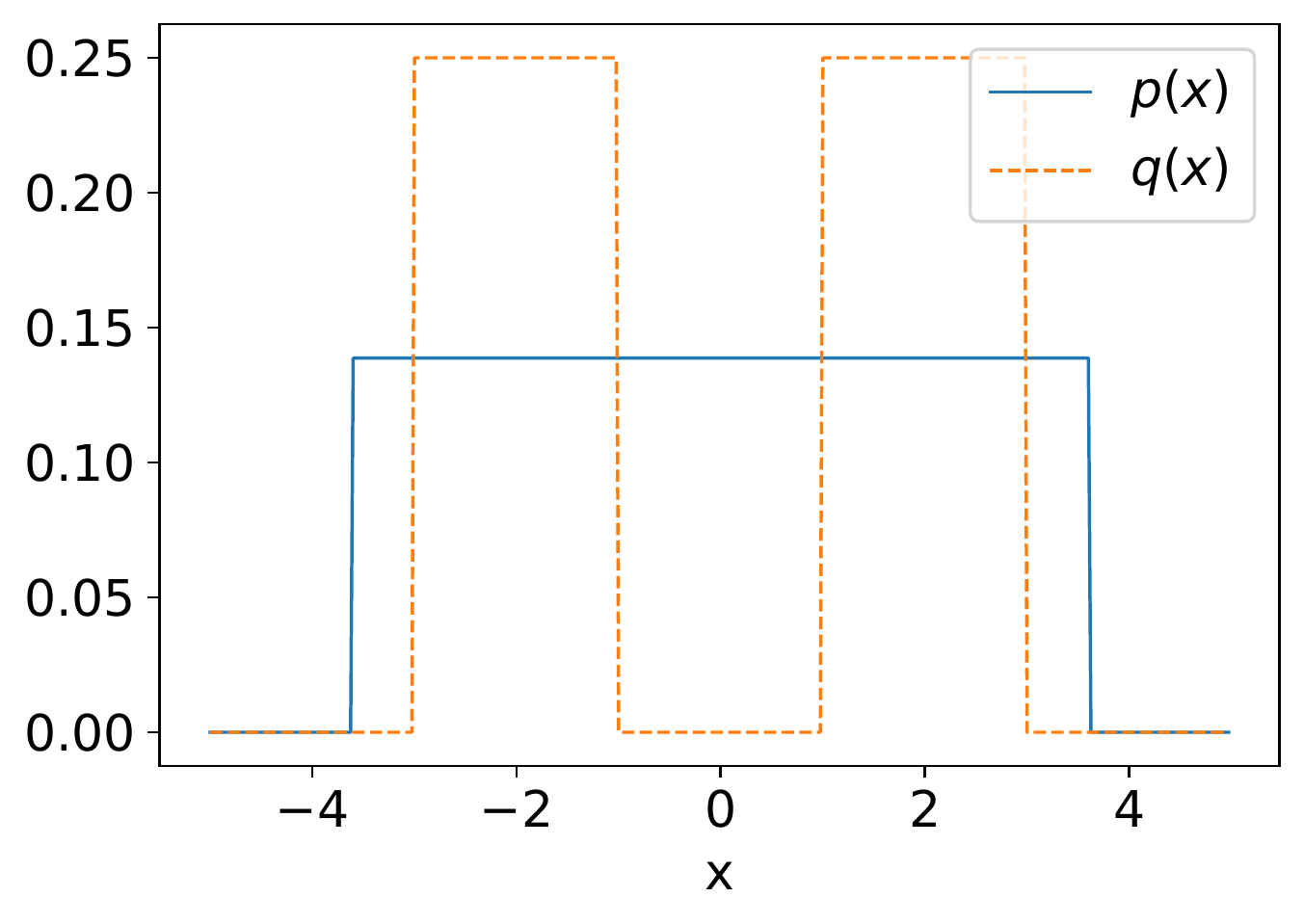}
    \caption{Probability density functions used for the experiment in Figure \ref{fig:laplaceII} of the main paper. Both distributions are symmetric and are constructed to have the same variance.}
    \label{fig:plot_uniform}
\end{figure}

\subsection{Type-I errors}\label{app:type_I}
To verify which methods are theoretically justified, i.e., control the Type-I error at a level $\alpha=0.05$, we run the following experiments, similar to the experiments in the main paper, where $P=Q$.
\begin{enumerate}
    \item \textsc{diff var} ($p=1$): $P = \mathcal{N}(0,1)$ and $Q=\mathcal{N}(0,1)$. 
    \item \textsc{MNIST} ($p=49$): We consider downsampled 7x7 images of the MNIST dataset \citep{lecun2010mnist}, where $P$ contains all the digits and $Q = P$.
    \item \textsc{Blobs} ($p=2$): A mixture of anisotropic Gaussians and $P=Q$.
\end{enumerate}

The results are in Figure \ref{fig:typeI}. All the methods except \textsc{naive} correctly control the Type-I error at a rate $\alpha = 0.05$ even for relatively small sample sizes. Note that all the described approaches rely on the asymptotic distribution. The critical sample size, at which it is safe to use, generally depends on the distributions $P$ and $Q$ and also the kernel functions. A good approach to simulating Type-I errors in in two-sample testing problems is to merge the samples and then randomly split them again. If the estimated Type-I error is significantly larger that $\alpha$, working with the asymptotic distribution is not reliable.

\begin{figure}[t]
\centering
\includegraphics[width=\linewidth]{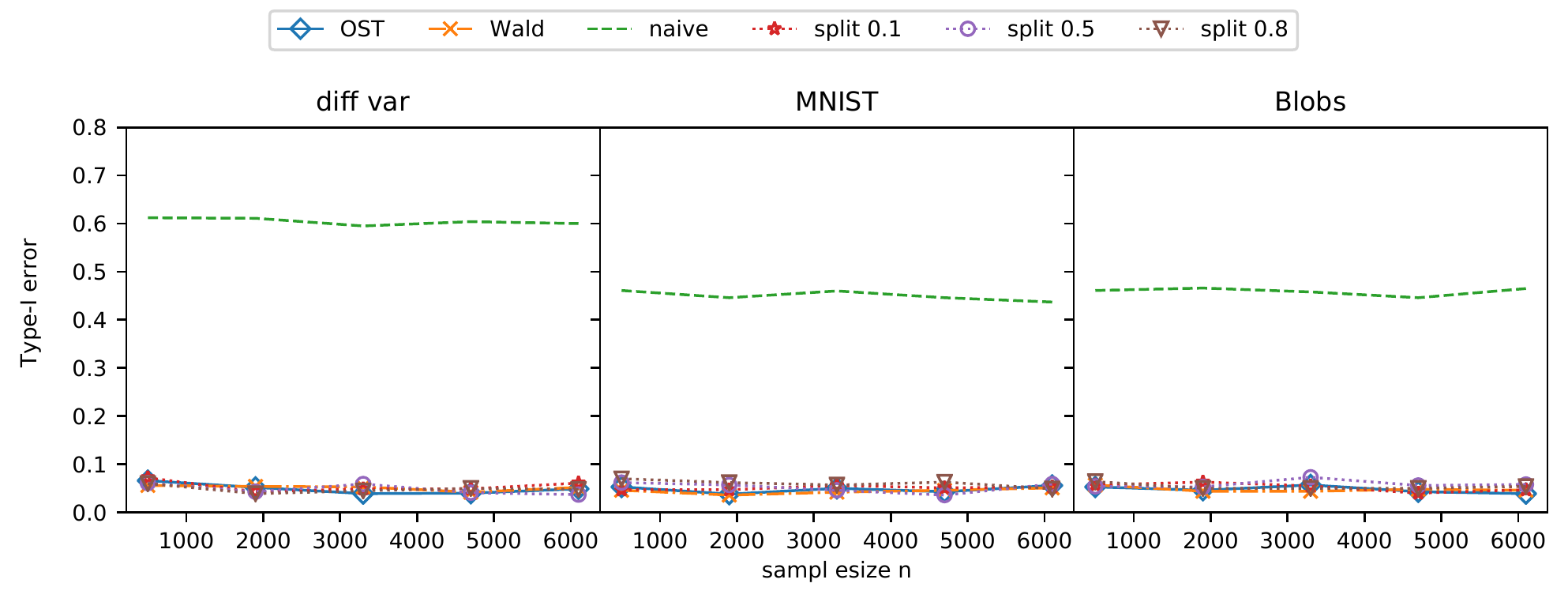}
\caption{Type-I errors for similar distributions as the one considered in the main paper. To simulate type-I errors we choose distributions $P=Q$ that are similar to the ones considered for the Type-II errors. We see that all well-calibrated methods reliably control the Type-I error at a rate $\alpha=0.05$, and conclude that working with the asymptotic distributions is well justified for the considered examples. The \textsc{naive} approach fails to control the error, as it overfits in the training phase without a correction in the testing phase.}
\label{fig:typeI}
\end{figure}

\subsection{Comparison of the constraints}\label{app:constraints}
In Section \ref{sec:continuous_general} we motivate to constrain the set of considered $\bm\beta$ to obey $\Sigma\bm\beta \geq \bm 0$, thus incorporating the knowledge $\bm\mu\geq \bm 0$. All our experiments suggest that this constraint indeed improves test power as compared to the general Wald test. In \citet{Gretton2012optimal} a different constraint was chosen. There $\bm\beta$ is constrained to be positive, i.e., $\bm\beta \geq \bm 0$. The motivation for their constraint is that the sum of positive definite (pd) kernel functions is again a pd kernel function \citep{Scholkopf01:LKS}. Thus, by constraining $\bm\beta \geq \bm 0$ one ensures that $k = \sum_{u=1}^d \beta_u k_u$ is also a pd kernel. While this is sensible from a kernel perspective, it is unclear whether this is smart from a hypothesis testing viewpoint. From the latter perspective we do not necessarily care whether or not $\bm\beta^*$ defines a pd kernel. Our approach instead was purely motivated to increase test power over the Wald test. In Figure \ref{fig:comparison} we thus compare the two different constraints to the Wald test on the examples that were also investigated in the main paper with $d=6$ kernels (again five Gaussian kernels and a linear kernel).

\begin{figure}
    \centering
    \includegraphics[width=\linewidth]{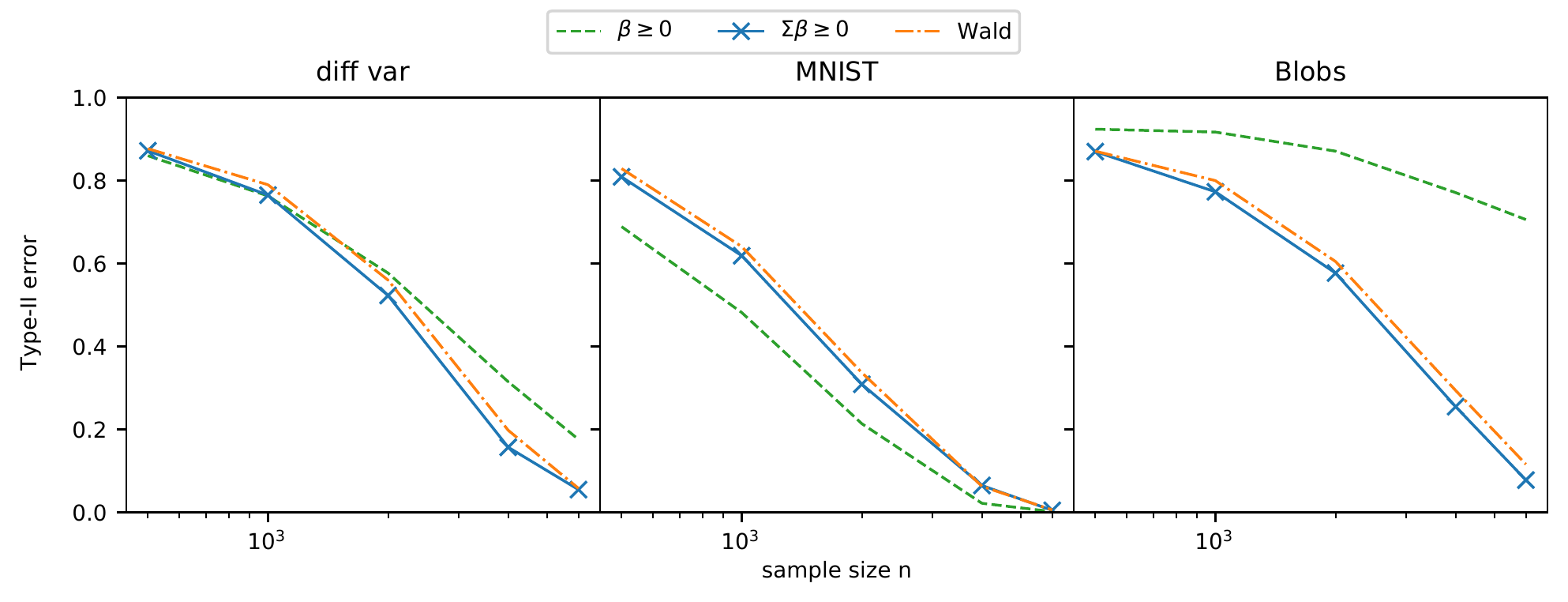}
    \caption{Comparison of the different constraints: In the main paper we argue that {\name} is a principled approach to constraint the class of considered tests, when $\bm\mu \geq \bm 0$ is guaranteed. \citet{Gretton2012optimal} suggested a different constraint $\bm\beta \geq \bm 0$. With Theorem \ref{thm:continuous}, we can also work with these constraints without data-splitting. The results suggest that indeed {\name} is a meaningful way to constrain the class of tests, as it consistently outperforms the Wald test. On the other hand the constraint suggested by \citet{Gretton2012optimal}, can only be seen as a heuristic. For some cases it performs better than the Wald test and the {\name}, but it can also perform worse.}
    \label{fig:comparison}
\end{figure}

From Figure \ref{fig:comparison} we observe that the positivity constraint of \citet{Gretton2012optimal} does not allow for general conclusions. Depending on the problem, the positivity constraint can both lead to higher or lower test power than the Wald test or tests with the constraint $\Sigma \bm\beta \geq \bm 0$. It will thus generally depend on the problem at hand which constraint is better. However, at least the approach we recommend ($\Sigma \bm\beta \geq \bm 0$) seems to guarantee a test power at least as high as the Wald test, whereas the positivity constraint can also be worse. As long as one has not a clear indication that the positivity constraint leads to better performance, we thus recommend the constraint $\Sigma \bm\beta \geq \bm 0$.

\subsection{Discrete selection from $T_\text{base}$}\label{app:discrete}

In this experiment, we use the same datasets and base kernels as for the experiment in the main paper. 
Instead of considering $T_\text{Wald}$ and $T_{\name}$, we consider $T_\text{base}$. 
We thus only compare to a data-splitting approach where also one of the base test statistics is selected. 
For completeness, we also include the \textsc{naive} approach, which again overfits for $d>1$. Note that the thresholds for $\tau_\text{base}$ can be computed with Corollary \ref{cor:selection_discrete} and do not rely on Theorem \ref{thm:continuous}. 
The results are shown in Figure \ref{fig:discrete}, again averaged over 5000 independent trials. 
In most of the cases, we observe that $\tau_\text{base}$ outperforms the data-splitting approaches. 
However, for the \textsc{MNIST} dataset and $d=2$, the splitting approach that uses $10\%$ for learning and $90\%$ for testing does perform slightly better.
Our attempt to explain this behavior lies in the truncation $\mathcal{V}^-$ of the conditional distribution. While for {\name}, we can show that $\mathcal{V}^- \leq 0$ (see proof of Theorem \ref{thm:continuous}), for Corollary \ref{cor:selection_discrete}, $\mathcal{V}^-$ cannot be bounded. If $\mathcal{V}^-$ is very large, the selected test is very conservative. We acknowledge that this is not a sufficient analysis of this phenomenon, but leave  a more theoretical treatment for future work. 
\begin{figure}
    \centering
    \includegraphics[width=1.\linewidth]{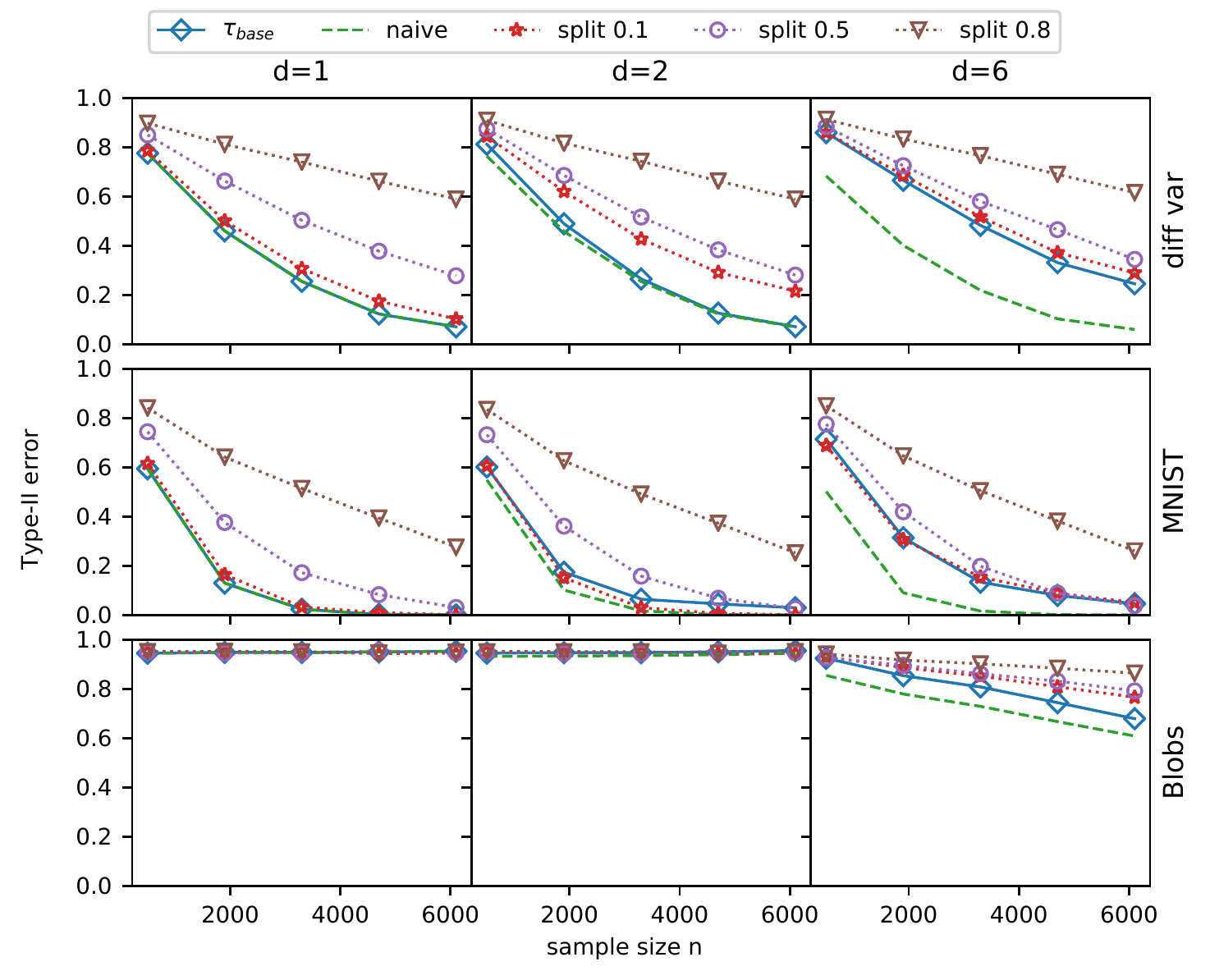}
    \caption{Type-II errors for discrete selection, i.e., the class of considered tests is $T_\text{base}$. The rows (columns) correspond to different datasets (sets of base kernels). Similar as in Figure \ref{fig:typeII}, our approach $\tau_\text{base}$ outperforms the splitting approaches in most cases. However, for the \texttt{MNIST} dataset and $d=2$ we see that the splitting approach with $10\%$ training and $90\%$ testing data (\textsc{split0.1}) performs better.}
    \label{fig:discrete}
\end{figure}

\section{Singular covariance matrices}\label{app:singular}
In the main paper we assumed that $\Sigma$ is strictly positive, i.e., non-singular. However, in practice, some eigenvalues of the covariance matrix can be sufficiently close to zero to cause numerical problems. In the case of the kernel two-sample test, this can happen if we consider kernels that are too similar and thus cause redundancy in our observations. In practice, this happens for example if we consider Gaussian kernels with too similar bandwidths on an easy problem. 

\textbf{Note on regularization:} One strategy to recover the numerical stability of the algorithm is to regularize the covariance matrix $\Sigma \rightarrow \Sigma + \lambda \mathop{I}$. Doing this indeed increases the numerical stability, since it leads to a well-behaved condition number. However, it also makes the whole approach more conservative, since the (artificially) increased variance decreases the value of the test statistic compared to the threshold. This leads to an increase of Type-II error and thus a loss of power. To evade this, we suggest the more elaborate strategy below.
 
Since $\Sigma$ is symmetric, there exists an orthonormal basis $\{v_i\}_{i\in [d]}$ and non-negative numbers $\{\lambda_i\}_{i\in [d]}$ such that
\begin{align*}
    \Sigma = \sum_{i=1}^d \lambda_i v_i v_i^\top.
\end{align*}
If $\Sigma$ is singular, we can assume WLOG that there exists $d_0\in [d]$ such that $\lambda_i = 0$ if $i\leq d_0$ and hence 
\begin{align*}
    \Sigma = \sum_{i = d_0 +1}^d \lambda_i v_i v_i^\top.
\end{align*}
Now if $v_i^\top \bm\tau \neq 0$ for some $i \in [d_0]$, we immediately know that $\bm\mu \neq \bm 0$ and could reject. In other words the signal-to-noise ratio along this direction is infinite. Thus, in the following we assume $v_i^\top \bm\tau = 0$ for all $i \in [d_0]$, and hence, $\sum_{i = d_0 +1}^d v_i v_i^\top \bm\tau = \bm\tau$. We can then rewrite the objective as follows
\begin{align*}
    \max_{\Sigma \bm\beta \geq \bm 0} \frac{ \bm \beta^\top \bm\tau}{(\bm\beta^\top \Sigma \bm\beta)^\frac{1}{2}} = \max_{\sum_{i = d_0 +1}^d \lambda_i v_i v_i^\top \bm\beta \geq \bm 0} \frac{ \bm \beta^\top \sum_{i = d_0 +1}^d v_i v_i^\top \bm\tau}{(\bm\beta^\top \sum_{i = d_0 +1}^d \lambda_i v_i v_i^\top \bm\beta)^\frac{1}{2}}.
\end{align*}
Now define $\bm\alpha := \sum_{i = d_0 +1}^d \lambda_i v_i v_i^\top \bm\beta$. Since $\Sigma$ is symmetric its pseudoinverse is given as $\Sigma^+ = \sum_{i = d_0 +1}^d \frac{1}{\lambda_i} v_i v_i^\top$ and we get
\begin{align*}
    \max_{\sum_{i = d_0 +1}^d \lambda_i v_i v_i^\top \bm\beta \geq \bm 0} \frac{ \bm \beta^\top \sum_{i = d_0 +1}^d v_i v_i^\top \bm\tau}{(\bm\beta^\top \sum_{i = d_0 +1}^d \lambda_i v_i v_i^\top \bm\beta)^\frac{1}{2}} = \max_{\bm\alpha \geq \bm 0}\frac{ \bm \alpha^\top \Sigma^+\bm\tau}{(\bm\beta^\top \Sigma^+ \bm\beta)^\frac{1}{2}}.
\end{align*}
Similar as in Remark \ref{rmk:generalization} we can define $\bm\rho := \Sigma^+\bm\tau$ and $\Sigma' = \Sigma^+$. However, in Theorem \ref{thm:continuous} we assumed that the covariance is not singular. Therefore in Theorem \ref{thm:continuous} we used $l = |\mathcal{U}|$, which corresponded to the rank of $\proj \Sigma \proj$ (see Appendix \ref{app:proof_continuous}). However, in the present case the rank of $\proj \Sigma^+ \proj$ does not equal the number of non-zero entries of $\bm\beta$. Therefore we use $l = \text{rank}(\proj \Sigma^+ \proj)$. With this we can apply Theorem \ref{thm:continuous} and get the conditional distribution under the null.

In practice, we have to treat the covariance matrix as singular if its condition number is below some threshold, as otherwise the numerical precision does not suffice to invert matrices faithfully.

\end{document}